\documentclass{article}

\usepackage[utf8]{inputenc}

\usepackage{amsmath, amssymb, amsthm, amsfonts}

\newtheorem*{definition}{Definition}
\newtheorem{prop}{Proposition}
\newtheorem{inv}{Invariant}
\newtheorem*{example}{Example} 

\newtheorem{theorem}{Theorem}
\newtheorem*{theorem*}{Theorem}
\newtheorem{corollary}{Corollary}

\DeclareMathOperator{\diag}{diag}

\usepackage{diagbox}
\usepackage{mathtools}
\usepackage{enumerate}
\usepackage{graphicx}
\usepackage{array}
\usepackage[margin=1in]{geometry}
\usepackage{tikz}
\usepackage{hyperref, doi}
\usepackage{epsf}

\usepackage[T1]{fontenc}    
\usepackage{float}

\usepackage{caption}
\usepackage{subcaption}
\usepackage{array}
\usepackage{algpseudocode,algorithm,algorithmicx}
\usepackage{makecell}
\usepackage{boldline}
\usepackage{pifont}
\usepackage[symbol]{footmisc}

\newcommand*\samethanks[1][\value{footnote}]{\footnotemark[#1]}

\title{Over-parameterized Neural Networks Implement Associative Memory}

\title{Overparameterized Neural Networks \\ Implement Associative Memory}
\author{Adityanarayanan Radhakrishnan \thanks{Laboratory for Information \& Decision Systems, and 
 Institute for Data, Systems, and Society, 
 Massachusetts Institute of Technology, Cambridge, MA 02139}
\and Mikhail Belkin \thanks{Department of Computer Science and Engineering, The Ohio State University,  Columbus, OH 43210}
\and Caroline Uhler \samethanks[1] 
}

\begin{document}

\maketitle

\begin{abstract}
Identifying computational mechanisms for memorization and retrieval of data is a long-standing problem at the intersection of machine learning and neuroscience.  Our main finding is that standard overparameterized deep neural networks trained using standard optimization methods implement such a mechanism for real-valued data.  Empirically, we show that: (1) overparameterized autoencoders store training samples as attractors, and thus, iterating the learned map leads to sample recovery; (2) the same mechanism allows for encoding sequences of examples, and serves as an even more efficient mechanism for memory than autoencoding.  Theoretically, we prove that when trained on a single example, autoencoders store the example as an attractor.  Lastly, by treating a sequence encoder as a composition of maps, we prove that sequence encoding provides a more efficient mechanism for memory than autoencoding.
\end{abstract}

\section{Introduction}
Developing computational models of \textit{associative memory}, a system which can recover stored patterns from corrupted inputs, is a long-standing problem at the intersection of machine learning and neuroscience.  An early example of a computational model for memory dates back to the introduction of Hopfield networks \cite{HopfieldNetwork, Little}.  Hopfield networks are an example of an \textit{attractor network}, a system which allows for the recovery of patterns by storing them as attractors of a dynamical system.  In order to write patterns into memory, Hopfield networks construct an energy function with local minima corresponding to the desired patterns.  To retrieve these stored patterns, the constructed energy function is iteratively minimized starting from a new input pattern until a local minimum is discovered and returned. 

While Hopfield networks can only store binary patterns, the simplicity of the model allowed for a  theoretical analysis of capacity \cite{HopfieldCapacity}.  In order to to implement a form of associative memory for more complex data modalities, such as images, the  idea of storing training examples as the local minima of an energy function was extended by several recent works \cite{MetaLearningEnergyModels, HintonDeepBeliefNets, HintonBookChapter, EnergyModelsImplicit, RuslanBoltzmann}.  Unlike Hopfield networks, these modern methods do not guarantee that a given pattern can be stored and typically lack the capacity to store patterns exactly; see e.g.~\cite{MetaLearningEnergyModels}.  

Our main finding is that standard over-parameterized neural networks trained using standard optimization methods can implement associative memory.  In contrast to energy-based methods, the storage and retrieval mechanisms are automatic consequences of training and do not require constructing and minimizing an energy function.  

\textbf{Interpolation Alone is Not Sufficient for Implementing Associative Memory.} While in recent machine learning literature (e.g., \cite{CloserLookAtMemorization, RethinkingGeneralization}) the term  memorization is often used interchangeably with interpolation, the ability of a model to perfectly fit training data, note that memorization is stronger and also requires a model to be able to recover training data. In general, interpolation does not guarantee the ability to recover training data nor does it guarantee the ability to associate new inputs with training examples. Fig.~\ref{fig: Figure 1}a shows an example of a function that interpolates training data, but does not implement associative memory: there is no apparent method to recover the training examples from the function alone.  On the other hand, Fig.~\ref{fig: Figure 1}b gives an example of a function that implements memory: the training examples are retrieved as the range of the function.

\begin{figure*}[!t]
\centering
    \begin{subfigure}[b]{\textwidth}
         \centering
     \includegraphics[height=2in]{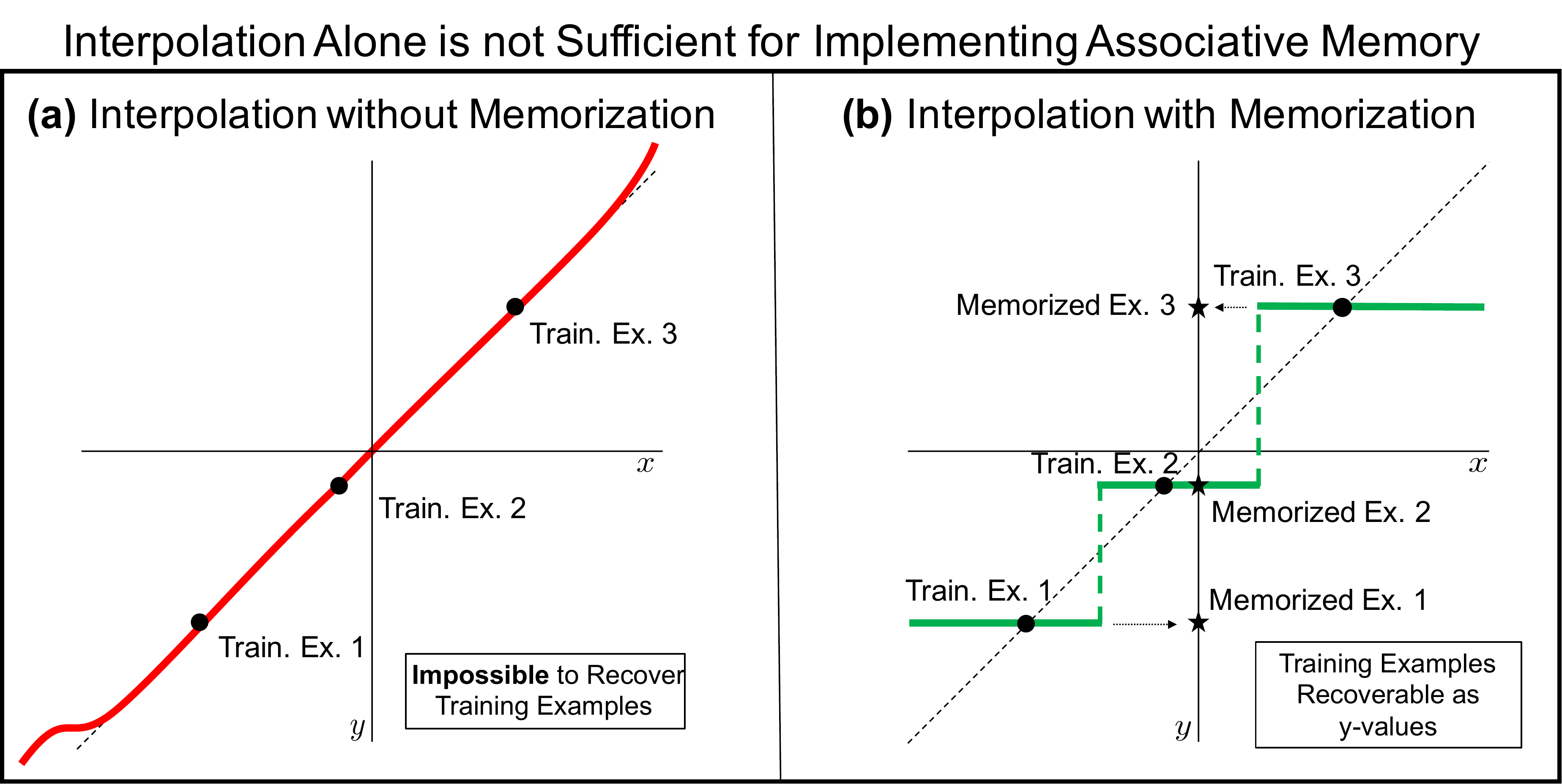}
     \end{subfigure}
     \begin{subfigure}[b]{\textwidth}
         \centering
     \includegraphics[height=2in]{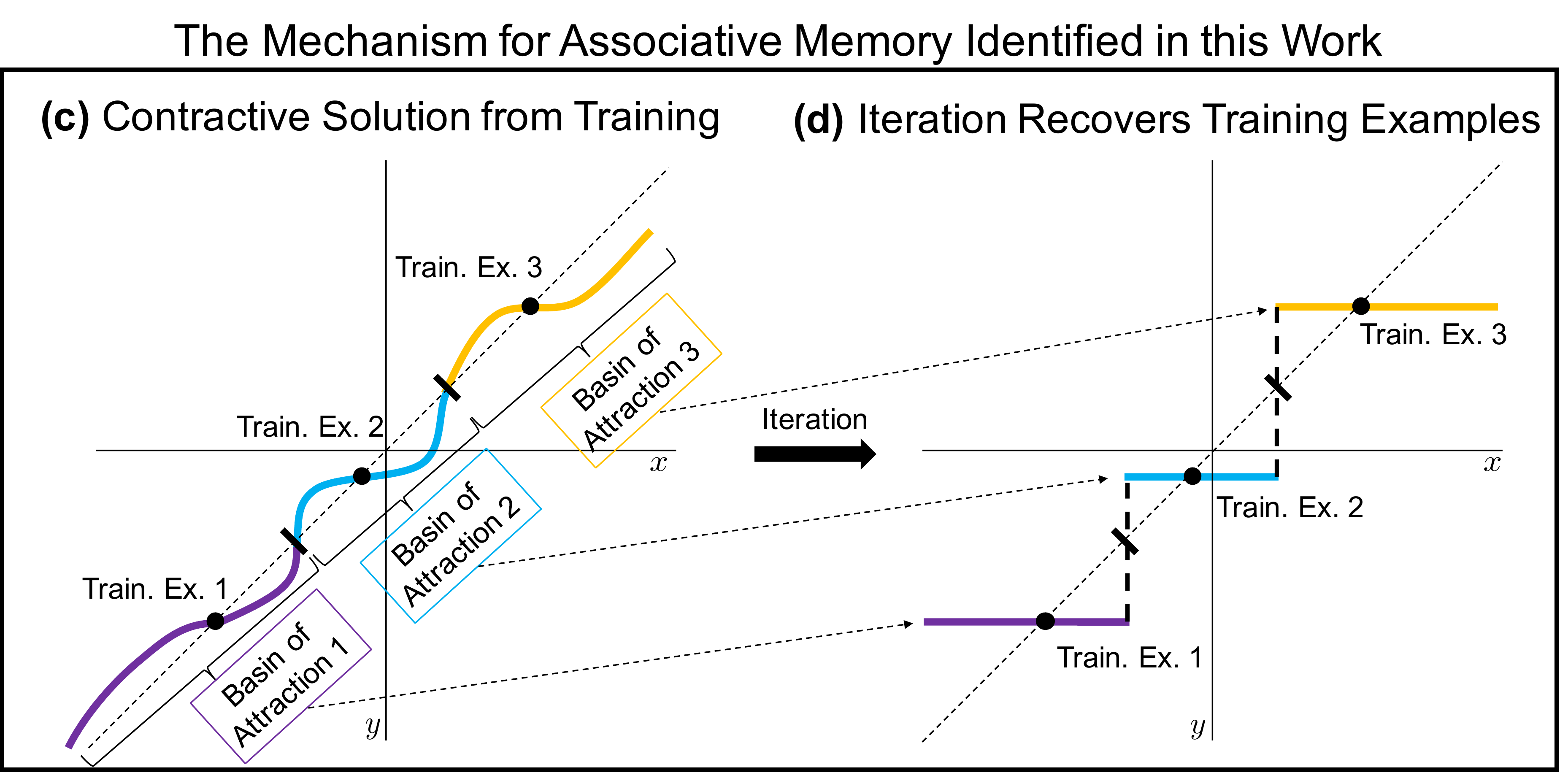}
     \end{subfigure}

\caption{The difference between associative memory and interpolation is described in (a, b); the mechanism identified in this work by which overparameterized autoencoders implement  associative memory is described in (c, d). Training examples are shown as black points, the identity function is shown as a dotted line and functions are represented using solid, colored lines.  (a) An example of a function that interpolates the training data but does not memorize training data: training data are not recoverable from just the function alone.  (b)  An example of a function that interpolates and memorizes training data: training data are recoverable as the range of the function.  (c) An example of an interpolating function for which the training examples are attractors; the basis of attraction are shown.  (d) Iteration of the function from (c) leads to a function that is piece-wise constant almost everywhere, with the training examples corresponding to the non-trivial constant regions.  The fact that the training examples are attractors implies that iteration provides a  retrieval mechanism.}
\label{fig: Figure 1}
\end{figure*}

While it has been observed (e.g.,~\cite{RethinkingGeneralization}) that over-parameterized networks can interpolate the training data, there is no apriori reason why it should be possible to recover the training data from the network.  In this work, we show that, remarkably, the retrieval mechanism also follows naturally from training: the examples can be recovered simply by iterating the learned map. A depiction of the memorization and retrieval mechanisms is presented in Fig.~\ref{fig: Figure 1}c and \ref{fig: Figure 1}d. More precisely, given a set of training examples $\{x^{(i)}\}_{i=1}^{n} \subset \mathbb{R}^{d}$ and an overparameterized neural network implementing a family of continuous functions $\mathcal{F} = \{f : \mathbb{R}^d \rightarrow \mathbb{R}^d\}$, we show that minimizing the following \textit{autoencoding} objective with gradient descent methods leads to training examples being stored as attractors:
\begin{equation}
\label{eq: autoencoder objective}
\arg\min_{f \in \mathcal{F}} \;\sum\limits_{i=1}^n \|f(x^{(i)}) - x^{(i)}\|^2
\end{equation}
Interestingly, attractors arise without any specific regularization to the above loss function.  We demonstrate this phenomenon by presenting a wealth of empirical evidence, including a network that stores $500$ images from ImageNet-64 \cite{PixelRNN} as attractors.  In addition, we present a proof of this phenomenon for over-parameterized networks trained on single examples.  

Furthermore, we show that a slight modification of the objective \eqref{eq: autoencoder objective} leads to an implementation of associative memory for sequences.  More precisely,  given a sequence of training examples $\{x^{(i)}\}_{i=1}^{n}\subset \mathbb{R}^d$, minimizing the following \textit{sequence encoding} objective with gradient descent methods leads to the training sequence being stored as a stable~limit~cycle:
\begin{equation}
\label{eq: sequence encoder objective}
\arg\min_{f \in \mathcal{F}} \;\sum\limits_{i=1}^n  \|f(x^{((i\hspace{-2mm}\mod n) + 1))}) - x^{(i)}\|^2.
\end{equation}

Multiple cycles can be encoded similarly (Appendix \ref{appendix: A}).  In particular, we provide several examples of networks storing video and audio samples as limit cycles.  Interestingly, these experiments suggest that sequence encoding provides a more efficient mechanism for memorization and retrieval of training examples than autoencoding.  By considering a sequence encoder as a composition of maps, we indeed prove that sequence encoders are more contractive to a sequence of examples than autoencoders~are~to~individual~examples. 

\section{Related Work}

Autoencoders \cite{AutoencodersBaldi} are commonly used for manifold learning, and the autoencoder architecture and objective (Eq.~\eqref{eq: autoencoder objective}) have been modified in several ways to improve their ability to represent data manifolds.  Two variations, contractive and denoising autoencoders, add specific regularizers to the objective function in order to make the functions implemented by the autoencoder contractive towards the training data \cite{DenoisingAutoencoders, ContractiveAutoencoders,  BengioContractiveAutoencoders}.  However, these autoencoders are typically used in the under-parameterized regime, where they do not have the capacity to interpolate (fit exactly) the training examples and hence cannot store the training examples as fixed points.   

On the other hand, it is well-known that over-parameterized neural networks can interpolate the training data when trained with gradient descent methods \cite{RethinkingGeneralization, DuOptimizesOverparameterizedNetworks, GlobalMinimaDeepNetworks, DuAdaptiveMethods}. As a consequence,  over-parameterized autoencoders can store training examples as fixed points. In particular, recent work empirically studied over-parameterized autoencoders in the setting with one training example \cite{IdentityCrisis}. 

In this paper, we take a dynamical systems perspective to study over-parameterized autoencoders and sequence encoders. In particular, we show that not only do over-parameterized autoencoders (sequence encoders) trained using standard methods store training examples (sequences) as fixed points (limit cycles), but that these fixed points (limit cycles) are \emph{attractors} (\emph{stable}), i.e., they can be recovered via iteration. While energy-based methods have also been shown to be able to  recall sequences as stable limit cycles \cite{KoskoLimitCycles, BuhmannLimitCycles}, the mechanism identified here is unrelated and novel: it does not require setting up an energy function and is a direct consequence of training an over-parameterized network. 

\section*{Background from Dynamical Systems}

We now introduce tools related to dynamical systems that we will use to analyze autoencoders and sequence encoders.

\vspace{2mm}

\noindent \textbf{Attractors in Dynamical Systems.} Let $f: \mathbb{R}^d \rightarrow \mathbb{R}^d$ denote the function learned by an autoencoder trained on a dataset $X = \{x^{(i)} \}_{i=1}^n\subset \mathbb{R}^{d}$. Consider the sequence $\{f^{k}(x)\}_{k \in \mathbb{N}}$ where $f^{k}(x)$ denotes $k$ compositions of $f$ applied to $x \in \mathbb{R}^d$. A point $x\in\mathbb{R}^d$ is a \textit{fixed point} of $f$ if  $f(x) = x$; in this case the sequence $\{f^{k}(x)\}_{k \in \mathbb{N}}$ trivially converges to $x$. 

Since overparameterized autoencoders interpolate the training data, it holds that $f(x^{(i)}) = x^{(i)}$ for each training example $x^{(i)}\in X$; hence all training examples are fixed points of~$f$.\footnote{To ensure $f(x^{(i)}) \approx x^{(i)}$, it is essential to train until the loss is very small; we used $\leq 10^{-8}$.}  We now formally define what it means for a fixed point to be an \emph{attractor} and provide a sufficient condition for this property. 

\begin{definition}
A fixed point $x^* \in \mathbb{R}^d$ is an \textbf{attractor} of $f:\mathbb{R}^d\to \mathbb{R}^d$ if there exists an open neighborhood, $\mathcal{O}$, of $x^*$, such that for any $x \in \mathcal{O}$, the sequence $\{f^k (x)\}_{k \in \mathbb{N}}$ converges to $x^*$ as $k \rightarrow \infty$.  The set $S$ of all such points is called the \textbf{basin of attraction} of $x^*$.  
\end{definition}

\begin{prop}
\label{prop: Proposition 1}
A fixed point $x^* \in \mathbb{R}^d$ is an attractor of a differentiable map $f$ if all eigenvalues of the  Jacobian of $f$ at $x^*$ are strictly less than $1$ in absolute value. If any of the eigenvalues are greater than $1$, $x^*$ cannot be an attractor.  
\end{prop}
\vspace{-1mm}

Proposition~\ref{prop: Proposition 1} is a well-known condition in the theory of dynamical systems (Chapter 6 of \cite{DynamicalSystem}).  The condition intuitively means that the function $f$ is ``flatter'' around an attractor $x^*$.  Since training examples are fixed points in over-parameterized autoencoders, from Proposition \ref{prop: Proposition 1}, it follows that a training example is an attractor if the maximum eigenvalue (in absolute value) of the Jacobian at the example is less than $1$.  Since attractors are recoverable through iteration, autoencoders that store training examples as attractors guarantee recoverability of these examples.  Energy-based methods also allow for verification of whether a training example is an attractor.  However, this requires checking the second order condition that the Hessian is positive definite at the training example, which is more computationally expensive than checking the first order condition from Proposition~\ref{prop: Proposition 1}.

\vspace{2mm}

\noindent \textbf{Discrete Limit Cycles in Dynamical Systems.}  Discrete limit cycles can be considered the equivalent of an attractor for sequence encoding, and a formal definition is provided below.

\begin{definition}
A finite set $X^* = \{x^{(i)}\}_{i=1}^{n} \subset \mathbb{R}^d$ is a \textbf{stable discrete limit cycle} of a smooth function $f : \mathbb{R}^d \rightarrow \mathbb{R}^d$ if: \textbf{(1)} $f(x^{(i)}) = x^{(i\mod n) + 1} ~ \forall i \in \{1, \ldots n\}$ ; \textbf{(2)} There exists an open neighborhood, $\mathcal{O}$, of $X^*$ such that for any $x \in \mathcal{O}$, $X^*$ is the limit set of $\{f^{k}(x)\}_{k=1}^{\infty}$.   
\end{definition}

The equivalent of Proposition \ref{prop: Proposition 1} for verifying that a finite sequence of points forms a limit cycle is provided below. 

\begin{prop}
\label{prop: Proposition 2}
Let network $f: \mathbb{R}^d \rightarrow \mathbb{R}^d$ be trained on a given sequence $x^{(1)}, \dots , x^{(n)}$ such that $f(x^{(i)})=x^{((i \mod n) + 1)}$. Then the sequence $\{x^{(i)} \}_{i=1}^{n}$ forms a stable discrete limit cycle if the largest eigenvalue of the Jacobian of $f^n(x^{(i)})$ is (in absolute value) less than 1 for any $i$.  
\end{prop}

This follows directly by applying Proposition \ref{prop: Proposition 1} to the map $f^n$, since $x^{(i)} = f^n(x^{(i)})$ and $f(x^{(i)})=x^{((i \mod n) + 1)}$.  Before presenting our results, we provide the following important remark. 
\vspace{2mm}

\noindent \textbf{Why the Emergence of Attractors in Autoencoders is Notable.} Proposition \ref{prop: Proposition 1} states that for a fixed point to be an attractor, all eigenvalues of the Jacobian at that point must be less than $1$ in absolute value. Since the number of eigenvalues of the Jacobian equals the dimension of the space, this means that the angle of the derivative is less than $\pi/4$ in \emph{every} eigendirection of the Jacobian. This is a highly restrictive condition, since intuitively, we expect the ``probability'' of such an event to be $1/2^d$. Hence, a fixed point of an arbitrary high-dimensional map is unlikely to be an attractor. Indeed, as we show in Corollary~1, fixed points of neural networks are not generally  attractors.
While not yet fully understood, the emergence and, indeed, proliferation of attractors in autoencoding is not due solely to architectures but to specific inductive biases of the training procedures.

\section{Empirical Findings}

\begin{figure*}[ht]
\centering
\includegraphics[height=3.5in]{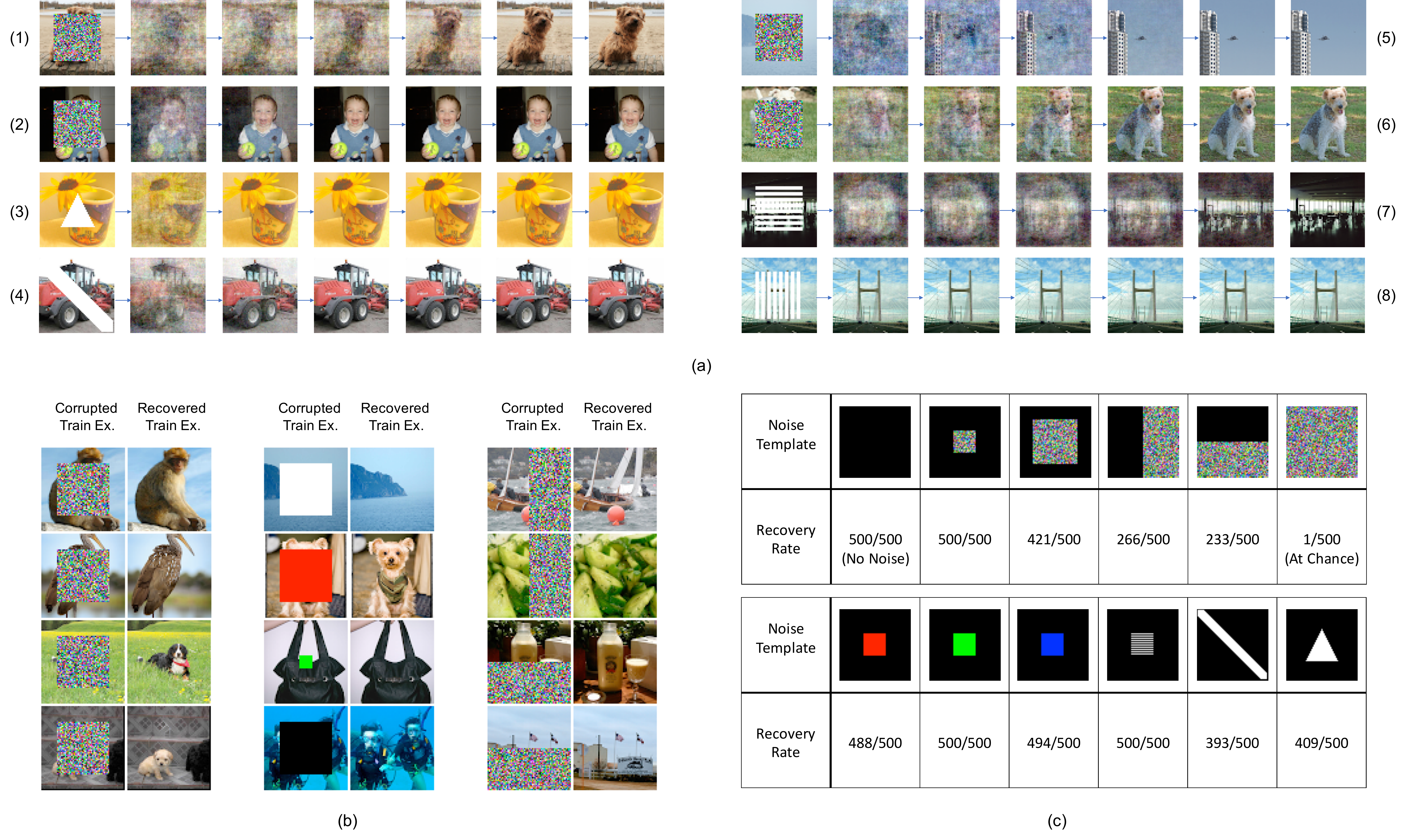}
\vspace{-0.3cm}
\caption{Example of an over-parameterized autoencoder storing $500$ images from ImageNet-64 as attractors after training to a reconstruction error of less than $10^{-8}$.  Architecture and optimizer details are provided in Appendix, Fig. \ref{appendix: fig 6}.  (a) By iterating the trained autoencoder on corrupted versions of training samples, individual training samples are recovered.  (b) Samples that are corrupted by uniform random noise or squares of varying color and size are recovered via iteration. (c) Fraction of samples recovered correctly from different noise applied to the training images. A sample is considered recovered when the error between the original sample and the recovered sample is less than $10^{-7}$.} 
\label{fig: Figure 2}
\end{figure*}

\noindent \textbf{Training Examples are Stored as Attractors in Over-parameterized Autoencoders.} 
In the following, we present a range of empirical evidence that attractors arise in autoencoders across common architectures and optimization methods. For details on the specific architecures and optimization schemes used for each experiment, see Appendix, Fig~\ref{appendix: fig 6}.

\vspace{2mm}
\noindent \textbf{Storing Images as Attractors.} In Fig.~\ref{fig: Figure 2}, we present an example of an over-parameterized autoencoder storing $500$ images from \mbox{ImageNet-64 \cite{PixelRNN}} as attractors.  This was achieved by training an autoencoder with depth $10$, width $1024$, and cosid nonlinearity \cite{Cosid} on 500 training examples using the Adam \cite{Adam} optimizer to loss $\leq 10^{-8}$.  We verified that all $500$ training images were stored as attractors by checking that the magnitudes of all eigenvalues of the Jacobian matrix at each example were less than $1$. Indeed, Fig.~\ref{fig: Figure 2}a demonstrates that iteration of the trained autoencoder map starting from corrupted inputs converges to individual training examples. 
A common practice for measuring recoverability of training patterns is to input corrupted versions of the patterns and verify that the system is able to recover the original patterns.  From Proposition \ref{prop: Proposition 1}, provided that a corrupted example is in the basin of attraction of the original example, iteration is guaranteed to converge to the original example.  In examples (5) and (6) Fig.~\ref{fig: Figure 2}a, the corrupted images are not in the basin of attraction for the original examples, and so iteration converges to a different (but contextually similar) training example. Fig.~\ref{fig: Figure 2}b provides further examples of correct recovery from corrupted images.  Fig.~\ref{fig: Figure 2}c presents a quantitative analysis of the recovery rate of training examples under various forms of corruption.  Overall, the recovery rate is remarkably high: even when $50\%$ of the image is corrupted, the recovery rate of the network is significantly higher than expected by chance.

Examples of autoencoders storing training examples as attractors when trained on $2000$ images from MNIST \cite{mnist-lecun1998} and $1000$ black-and-white images from CIFAR10 \cite{CIFAR10} are presented in the Appendix Fig.~\ref{appendix: fig 7}, Fig.~\ref{appendix: fig 8}, respectively.  The MNIST autoencoder presented in Appendix, Fig.~\ref{appendix: fig 7} stores $2000$ training examples as attractors.  Note that one iteration of the learned map on test examples can look similar to the identity function, but in fact, iterating until convergence yields a training example (see Appendix Fig.~\ref{appendix: fig 7}).
\vspace{2mm}

\noindent \textbf{Spurious Attractors.} While in these examples, we verified that the training examples were stored as attractors by checking the eigenvalue condition, there could be \textit{spurious attractors}, i.e. attractors other than the training examples.  In fact, spurious attractors are known to exist for Hopfield networks \cite{SpuriousAttractorsHopfield}.  To investigate whether there are additional attractors outside of the training examples, we iterated the map from sampled test images and randomly generated images until convergence. More precisely, we declared convergence of the map at iteration $k$ for some image $x$ when $\| f^{k+1}(x) - f^{k}(x) \|_2 < 10^{-8}$ and concluded that $f^{k}(x)$ had converged to the training example $x^{(i)}$ if $\|f^{k}(x) - x^{(i)} \|_2 < 10^{-7}$.

In general, spurious attractors can exist for over-parameterized autoencoders, and we provide examples in the Appendix Fig.~\ref{appendix: fig 9}.  However, remarkably, for the network presented in Fig.~\ref{fig: Figure 2}, we could not identify any spurious attractors even after iterating the trained map from $40,000$ test examples from ImageNet-64, $10,000$ examples of uniform random noise, and $10,000$ examples of Gaussian noise with variance $4$.

\vspace{2mm}

\noindent \textbf{Attractors Arise across Architectures, Training Methods \& Initialization Schemes.}  We performed a thorough analysis of the attractor phenomenon identified above across a number of common architectures, optimization methods, and initialization schemes. Starting with fully connected autoencoders, we analyzed the number of training examples stored as attractors when trained on $100$ black and white images from CIFAR10 \cite{CIFAR10} under the following nonlinearities, initializations, and optimization methods:
\begin{itemize}
    \item \textit{Nonlinearities}: ReLU, Leaky Relu, SELU, cosid ($\cos x - x$), Swish \cite{Cosid, Swish, LeakyReLU, Selu}, and sinusoidal ($x + (\sin10x)/5)$.
    \item \textit{Optimization Methods}: Gradient Descent (GD), GD with momentum, GD with momentum and weight decay, RMSprop, and Adam (Ch.~8 of \cite{goodfellow2016deep}).
    \item \textit{Initialization Schemes}: Random uniform initialization, namely $U[-a, a]$, per weight for $a\in\{0.01, 0.02, $ \\ $0.05, 0.1, 0.15\}$.  These initialization schemes subsume the PyTorch (Version 0.4) default, Xavier initialization, and Kaiming initialization~\cite{PyTorch, KaimingInit, Xavier}.
\end{itemize}

\begin{figure*}[!t]
\centering
 \includegraphics[height=1.37in]{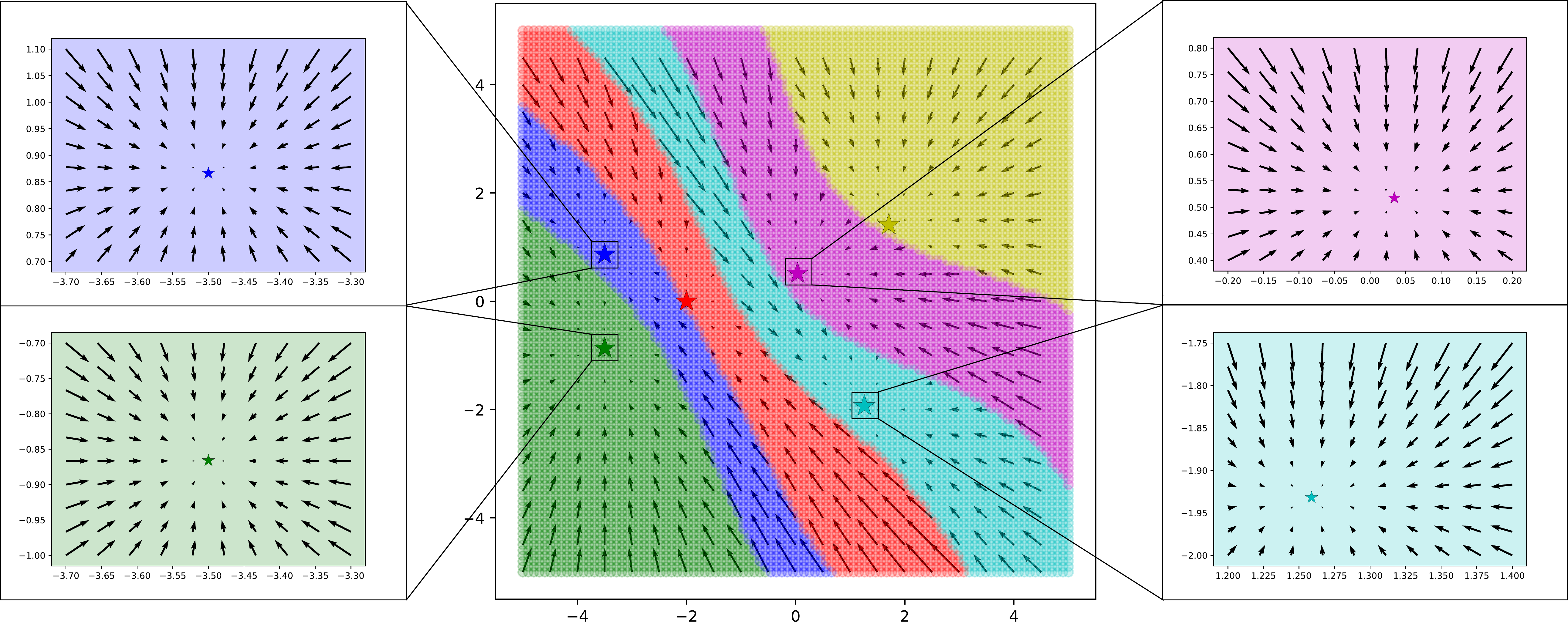}
\caption{Example of an over-parameterized autoencoder in the 2-dimensional setting storing training examples (represented as stars) as attractors.  Basins of attraction for each sample are colored by sampling $10,000$ points in a grid around the training examples, taking the limit of the iteration for each point, and assigning a color to the point based on the training example indicated by the limit.  The vector field indicates the direction of motion given by iteration, and the inserts indicate that iteration leads to training examples for all points in an open set around each example.  
\label{fig: Figure 3}
}
\end{figure*}

In Appendix, Fig.~\ref{appendix: fig 15} and \ref{appendix: fig 16}, we provide the number of training examples stored as attractors for all possible combinations of (a) nonlinearity and optimization method listed above; and (b) nonlinearity and initialization scheme listed above.  These tables demonstrate that attractors arise in all settings for which training converged to a sufficiently low loss within $1,000,000$ epochs.  In Appendix Figures \ref{appendix: fig 15}, \ref{appendix: fig 16} and Appendix \ref{appendix: F}, we also present examples of convolutional and recurrent networks that store training examples as attractors, thereby demonstrating that this phenomenon is not limited to fully connected networks and occurs in all commonly used network architectures.  
\vspace{2mm}

\noindent \textbf{Visualizing Attractors in 2D.} In order to better understand the attractor phenomenon, we present an example of an over-parameterized autoencoder storing training examples as attractors in the 2D setting, where the basins of attraction can easily be visualized (Fig.~\ref{fig: Figure 3}).  We trained an autoencoder to store $6$ training examples as attractors. Their basins of attraction were visualized by iterating the trained autoencoder map starting from $10,000$ points on a grid until convergence. The vector field indicates the direction of motion given by iteration. Also in this experiment, we found no spurious attractors.  Each training example and corresponding basin of attraction is colored differently. Interestingly, the example in Fig.~\ref{fig: Figure 3} shows that the metric learned by the autoencoder to separate the basins of attraction is not Euclidean distance, which would be indicated by a Voronoi diagram. 

\vspace{2mm}

\begin{figure*}[!t]
    \centering
     \includegraphics[height=2.4in]{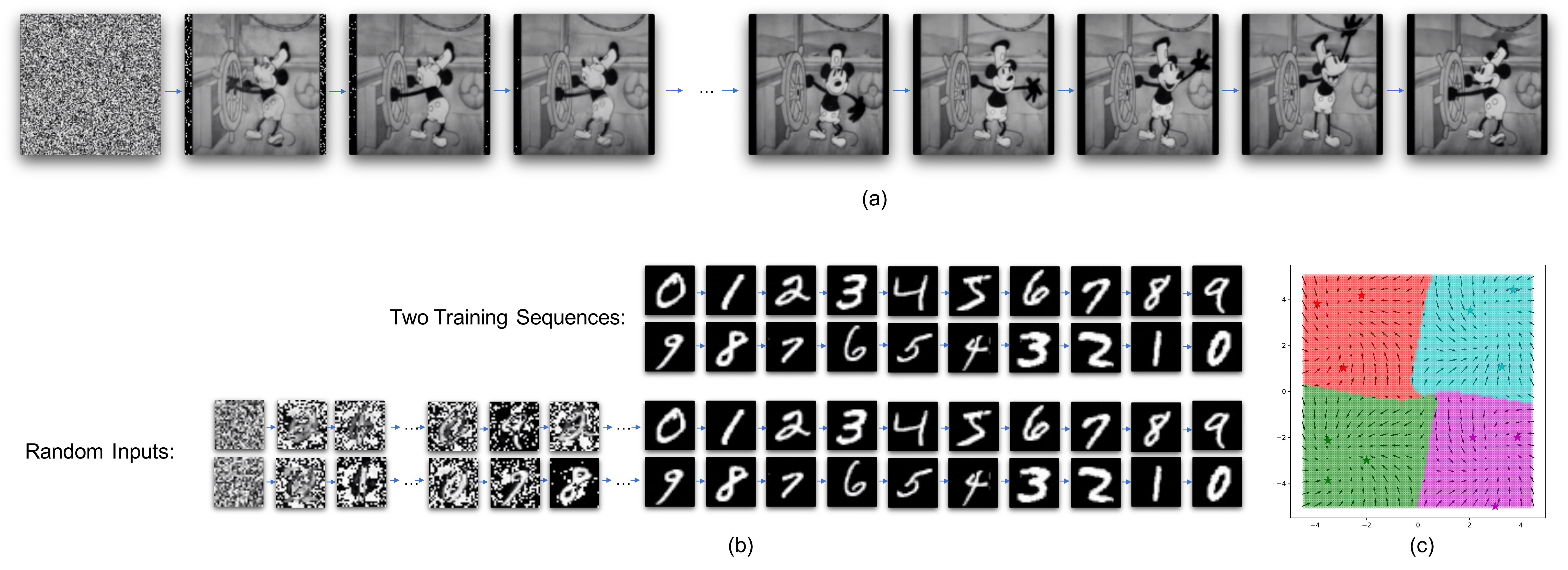}
     \vspace{-0.3cm}
    \caption{Examples of over-parameterized sequence encoders storing training sequences as limit cycles.  Architecture and optimizer details are provided in Appendix, Fig.~\ref{appendix: fig 6}.  (a) When trained on $389$ frames of size $128 \times 128$ from the Disney film ``Steamboat Willie'', the entire movie was stored as a limit cycle.  Hence, iteration from random noise leads to recovery of the entire sequence.  (b) When trained on two sequences of length $10$ from MNIST, each sequence was stored as a limit cycle.  Hence iteration from random noise leads to the recovery of each individual sequence.  (c) Visualization of the basins of attraction for a sequence encoder storing 4 sequences as limit cycles in the 2-dimensional setting.  The vector field indicates the direction of motion given by iteration.}
    \label{fig: Figure 4}
\end{figure*}

\noindent \textbf{Over-parameterized Sequence Encoders Store Training Examples as Stable Limit Cycles and are More Efficient at Memorizing and Retrieving Examples than Autoencoders.}  We have thus far analyzed the occurrence of attractors in over-parameterized autoencoders.  In this section, we demonstrate via various examples that by modifying the autoencoder objective to encode sequences (Eg.~[\ref{eq: sequence encoder objective}]), we can implement a form of associative memory for sequences. For details on the specific architectures and optimization schemes used for each experiment, see Appendix Figure~\ref{appendix: fig 6}. 
\vspace{2mm}

\noindent \textbf{Storing Sequences as Limit Cycles.}
We trained a network to encode 389 frames of size $128 \times 128$ from the Disney film ``Steamboat Willie'' by mapping frame $i$ to frame $i+1 \mod 389$.  Fig.~\ref{fig: Figure 4}(a) and the attached video\footnote{Located at: \url{https://github.com/uhlerlab/neural_networks_associative_memory}} show that iterating the trained network starting from random noise yields the original video. 

As a second example, we encoded two 10-digit sequences from MNIST: one counting upwards from digit 0 to 9 and the other counting down from digit 9 to 0. The maximal eigenvalues of the Jacobian of the trained encoder composed $10$ times is 0.0034 and 0.0033 for the images from the first and second sequence, respectively.  Hence by Proposition \ref{prop: Proposition 2}, both sequences form limit cycles. Indeed, as shown in Fig.~\ref{fig: Figure 4}(b), iteration from Gaussian noise leads to the recovery of both training sequences.  

Finally, in Fig.~\ref{fig: Figure 4}(c), we visualized the vector field and basins of attraction for four cycles in the 2-dimensional setting.  Unlike autoencoding where points near a training example are pushed towards it via iteration, the points now move following the cycles.  In Appendix \ref{appendix: G}, we also trained a sequence encoder that stores $10$ seconds of speech as a limit cycle.  The attached audio file demonstrates that iterating the trained network from random noise recovers the original audio. 

\begin{figure*}[!t]
\centering
    \includegraphics[height=1.7in]{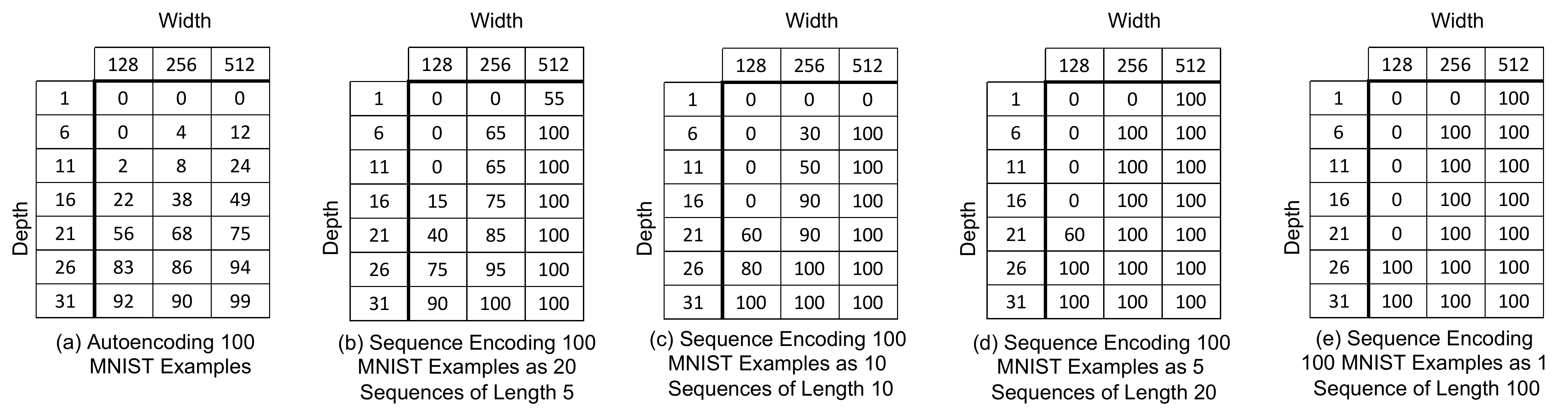}
    \vspace{-0.3cm}
\caption{Sequence encoders are more efficient at implementing associative memory than autoencoders.  Number of training examples recovered are out of 100; architecture and optimizer details are provided in the Appendix, Fig.~\ref{appendix: fig 6}. (a) Number of recovered images when autoencoding 100 examples from MNIST individually; a network of depth 31 and width 512 recovers $99$ images out of 100. (b)-(e) Sequence encoding the same 100 MNIST examples as sequences of different lengths improves the recovery rates; in particular, a network of depth 1 and width 512 recovers the full 10 images when encoded as 5 sequences of length 20 (d) or 1 sequence of length 100 (e).}
\label{fig: Figure 5}
\end{figure*} 

\vspace{2mm}

\noindent \textbf{Efficiency of Sequence Encoding.} In Fig.~\ref{fig: Figure 5}, we analyze the network sizes (width and depth) needed to store and retrieve $100$ training images from MNIST using fully connected autoencoders and sequence encoders.
Interestingly, our experiments indicate that memorization and retrieval of training examples can be performed more efficiently through sequence encoding than autoencoding. In particular, Fig.~\ref{fig: Figure 5}(a) shows the number of training examples (out of 100) that are attractors for different width and depth of the network. Note that a depth of $31$ and width of $512$ is needed to store almost all (99) training examples. If we instead encode the same data using $20$ sequences of length $5$, all $20$ sequences (and thus all $100$ examples) can be recovered using a much smaller network with a depth of $6$ and $512$ hidden units per layer (Fig.~\ref{fig: Figure 5}(b)). Extending this idea further (Fig.~\ref{fig: Figure 5}(c)-(e)), if we chain all $100$ examples as a single sequence, the entire sequence is stored using a network with only $1$ hidden layer and $512$ hidden units.  
\vspace{2mm}

\noindent \textbf{Increasing Depth and Width leads to more Attractors/Limit Cycles.}
The experiments in Fig.~\ref{fig: Figure 5} indicate that increasing network depth and width leads to an increase in the number of training examples / sequences stored as attractors / limit cycles. For over-parameterized autoencoders, this implies that the maximum eigenvalue of the Jacobian is less than $1$ for a greater number of training examples upon increasing network depth and width (see Proposition \ref{prop: Proposition 1}),  i.e., the network becomes more \textit{contractive} around the training examples. Indeed, by analyzing the histogram of the maximum eigenvalue of the Jacobian at each of the training examples, we observed that as network depth and width increases, the mode of these histograms shifts closer to zero (Appendix Fig.~\ref{appendix: fig 12}). Additionally, when considering the distribution of the top $1\%$ of Jacobian eigenvalues, we find that as network width increases, the variance of the distribution of Jacobian eigenvalues decreases, and when depth increases, the mode of the distribution shifts closer to zero (Appendix Fig.~\ref{appendix: fig 13}). In the following, we prove this phenomenon for a single training example, i.e., we prove that autoencoders trained on a single example become more contractive at the training example with increasing depth and width.

\section{Theoretical Analysis of Special Cases}

We now provide theoretical support for our empirical findings. Complete proofs are given in Appendix \ref{appendix: B}-\ref{appendix: E}.
\vspace{2mm}

\noindent \textbf{Proof that when trained on a single example, over-parameterized autoencoders store the example as an attractor.} We outline the proof for the 1-hidden layer setting. The complete proof for the multi-layer setting is given in the Appendix \ref{appendix: C}. 

Let $f(z) = W_1 \phi(W_2 z)$ represent a 1-hidden layer autoencoder with elementwise nonlinearity $\phi$ and weights  $W_1 \in \mathbb{R}^{k_0 \times k}$ and $W_2 \in \mathbb{R}^{k \times k_0}$, applied to $z \in \mathbb{R}^{k_0}$.  We analyze the function learned by gradient descent with learning rate $\gamma$ by minimizing the following autoencoding loss on 1 training example $x$:
\begin{align}
\label{eq_1_train}
    \mathcal{L}(x, f) = \frac{1}{2} \| x - f(x) \|_2^2.
\end{align}
Let $W_1^{(t)}, W_2^{(t)}$ denote the values of the weights after $t$ steps of gradient descent. To prove that $x$ is an attractor of $f$ after training, we solve for $W_1^{(\infty)}, W_2^{(\infty)}$ and compute the top eigenvalue of the Jacobian of $f$ at $x$ (denoted $\lambda_1(\mathbf{J}(f(x)))$).  

In order to solve for $W_1, W_2$, we first identify two invariants of gradient descent (proved in Appendix \ref{appendix: B}):  
\vspace{1mm}

\hspace{-4mm}\textbf{Invariant 1: } If $W_1$ and $W_2$ are initialized to be rank $1$ matrices of the form $x {u^{(0)}}^T$ and $v^{(0)} x^T$ respectively, 
then $W_1^{(t)} = x{u^{(t)}}^T$  and $W_2^{(t)} = v^{(t)} x^T$ for all time-steps $t>0$. 
\vspace{1mm}

\hspace{-4mm}\textbf{Invariant 2:} If, in addition, all weights in each row of $W_1$ and $W_2$ are initialized to be equal,  they remain equal throughout training.
\vspace{.5mm}

Invariant 1 implies that autoencoders trained on 1 example produce outputs that are multiples of the training example.  Generalizing this result, in Appendix \ref{appendix: D}, we prove that autoencoders trained on multiple examples produce outputs in the span of the training data.  Invariant 2 reduces gradient descent dynamics to the 1-dimensional setting. Using the Invariants 1 and 2 combined with gradient flow (i.e. taking the limit as the learning rate $\gamma \rightarrow 0$), we can solve for $W_1^{(\infty)}$ and $W_2^{(\infty)}$.

\begin{theorem}
\label{thm: closed form 1 hidden layer}
Let $f(z) = W_1 \phi(W_2 z)$ denote a 1-hidden layer network with elementwise nonlinearity $\phi$ and weights $W_1 \in \mathbb{R}^{k_0 \times k}$ and $W_2 \in \mathbb{R}^{k \times k_0}$, applied to $z \in \mathbb{R}^{k_0}$. Let $x \in \mathbb{R}^{k_0}$ be a training example with $\|x\|_2 = 1$. Assuming $\frac{\phi(z)}{\phi'(z)} < \infty ~\forall z \in \mathbb{R}$, then under Invariants 1 and 2, gradient descent with learning rate $\gamma\to 0$ applied to minimize the autoencoding loss in Eq.~\eqref{eq_1_train} leads to a rank 1 solution $W_1^{(\infty)} = xu^T$ and $W_2^{(\infty)} = vx^T$ with $u, v \in \mathbb{R}^{k}$ satisfying:
\begin{align*}
  \frac{u_i^2 - {u_i^{(0)}}^2}{2} &= \int_{v_i^{(0)}}^{v_i} \frac{\phi(z)}{\phi'(z)} dz \quad \textrm{and} \quad
  u_i \phi(v_i) = \frac{1}{k},
\end{align*}
and $u_i = u_j, v_i = v_j$ for all $i, j \in [k]$, where $u^{(0)}$ and $v^{(0)}$ are such that $W_1^{(0)} = x {u^{(0)}}^T$ and  $W_2^{(0)} = v^{(0)} x^T$. 
\end{theorem}

Theorem \ref{thm: closed form 1 hidden layer} allows us to compute the top eigenvalue of the Jacobian at $x$, denoted by $\lambda_1(\mathbf{J}(f(x)))$. 

\begin{theorem}
\label{thm: Top Eigenvalue}
Under the setting of Theorem \ref{thm: closed form 1 hidden layer}, it holds that
\begin{align*}
\lambda_1(\textbf{J}(f(x))) = \frac{\phi'(v_i)v_i}{\phi(v_i)}. 
\end{align*} 
\end{theorem}

Using Theorem \ref{thm: Top Eigenvalue}, we can explicitly determine whether a training example $x$ is an attractor, when given a nonlinearity $\phi$, initial values for $u^{(0)}$ and $v^{(0)}$, and the width of the network $k$.  We note that for all non-piecewise nonlinearities used thus far, we can make any training example an attractor by selecting values for $u^{(0)}$, $v^{(0)}$ and $k$ appropriately.  

\vspace{0.2cm}

\noindent\emph{Example.} Let $x$ be a training example in $\mathbb{R}^{k_0}$.  Suppose $\phi(z) = \frac{1}{1+e^{-z}}$ for $z \in \mathbb{R}$, $k=2$, and $u_i^{(0)} = v_i^{(0)} = 1$ for all $i$.  Then by Theorems \ref{thm: closed form 1 hidden layer} and \ref{thm: Top Eigenvalue}, it holds after training that
\begin{align*}
    &\frac{u_i^2 - 1}{2} = \int_{1}^{v_i} \left(\frac{1}{1 - \phi(z)}\right) dz \quad \textrm{and} \quad \frac{u_i}{1 + e^{-v_i}}  = \frac{1}{2}
\end{align*}
with $u_i \approx .697$, $v_i \approx .929$ and $\lambda_1(\mathbf{J}(f(x))) \approx .263$. Since $\lambda_1(\mathbf{J}(f(x))) < 1$, $x$ is an attractor. We also confirmed this result (up to third decimal place) experimentally by training a network using gradient descent with learning rate $10^{-4}$. 

\vspace{0.2cm}

Importantly,  the analysis of Theorem~\ref{thm: Top Eigenvalue} implies that attractors arise as a result of training and are not simply a consequence of interpolation by a neural network with a certain architecture; see the following corollary. 

\vspace{0.2cm}
\noindent\textbf{Corollary 1. }  \emph{Let $x \in \mathbb{R}^{k_0}$ with $\|x\|_2 = 1$ and  $f(z) = xu^T \phi(vx^T z)$, where $u, v \in \mathbb{R}^k$ and $\phi$ is a smooth element-wise nonlinearity with $\frac{\phi'(z)}{\phi(z)} < \infty$ for all $z \in \mathbb{R}$, $\left|\frac{\phi'(z) z}{\phi(z)}\right| > 1$ for $z$ in an open interval $\mathcal{O} \subset \mathbb{R}$.  Then there exist infinitely many $v \in \mathbb{R}^k$, such that $f(x) = x$ and $x$ is not an attractor for $f$.}

The condition, $\left|\phi'(z) z / \phi(z)\right| > 1$ for $z$ in an open interval, holds for all smooth non-linearities considered in this paper. The proof is presented in Appendix \ref{appendix: B}.  

We note that while the linear setting with $\phi(z) = z$ has been studied extensively using gradient flow \cite{MatrixFactorizationNIPS2017, Balancedness, ApproximateBalancedness}, our results extend to the non-linear setting and require novel tools. 

\textbf{Remarks on the Multiple Sample Setting.} While we extend Invariant 1 to the multiple example setting in Appendix~\ref{appendix: D}, a similar extension of Invariant 2 is required in order to generalize Theorem~1 to multiple examples.  We believe such an extension may be possible for orthonormal training examples.  Under random initialization, it may be possible to prove the attractor phenomenon by analyzing autoencoders in the Neural Tangent Kernel (NTK) regime~\cite{NTK}. However, the disadvantage of such an analysis is that it relies on computing a closed form for the NTK in the limiting case of network width approaching infinity.  On the other hand, Theorem 1 holds for a general class of non-linearities and for finite width and depth.


\textbf{Remarks on Similarity to Power Iteration.} The attractor phenomenon identified in this work appears similar to that of Fast Independent Component Analysis~\cite{FastICA} or more general nonlinear power iteration~\cite{EigsDecompFunctions}, where every ``eigenvector" (corresponding to a training example in our setting) of a certain iterative map has its own basin of attraction. In particular, increasing network depth may play a similar role to increasing the number of iterations in those methods.  While the mechanism may be different, understanding this connection is an important direction for future work.

\vspace{2mm}

\noindent \textbf{Proof that sequence encoding provides a more efficient mechanism for memory than autoencoding by analyzing sequence encoders as a composition of maps.} We start by generalizing Invariants 1, 2, and Theorem \ref{thm: closed form 1 hidden layer} to the case of training a network to map an example $x^{(i)} \in \mathbb{R}^{k_0}$ to an example $x^{(i+1)} \in \mathbb{R}^{k_0}$ as follows.

\begin{theorem}
\label{thm: closed from sequence encoding}
Let $f(z) = W_1 \phi(W_2 z)$ denote a 1-hidden layer network with elementwise nonlinearity $\phi$ and weights $W_1 \in \mathbb{R}^{k_0 \times k}$ and $W_2 \in \mathbb{R}^{k \times k_0}$, applied to $z \in \mathbb{R}^{k_0}$. Let $x^{(i)}, x^{(i+1)} \in \mathbb{R}^{k_0}$ be training examples with $\|x^{(i)}\|_2 = \|x^{(i+1)}\|_2 = 1$. Assuming that $\frac{\phi(z)}{\phi'(z)} < \infty ~\forall z \in \mathbb{R}$ and there exist $u^{(0)}, v^{(0)} \in \mathbb{R}^{k}$ such that $W_1^{(0)} = x^{(i+1)}{u^{(0)}}^T$ and $W_2^{(0)} = v^{(0)}{x^{(i)}}^T$ with $u_i^{(0)} = u_j^{(0)}, v_i^{(0)} = v_j^{(0)} ~\forall i, j \in [k]$, then
gradient descent with learning rate $\gamma\to 0$ applied to minimize
\begin{align}
\label{loss_sequence}
    \mathcal{L}(x, f) = \frac{1}{2} \| x^{(i+1)} - f(x^{(i)}) \|_2^2
\end{align}
leads to a rank 1 solution
$W_1^{(\infty)} = x^{(i+1)} u^T$ and $W_2^{(\infty)} = v {x^{(i)}}^T$ with $u, v \in \mathbb{R}^{k}$ satisfying
\begin{align*}
  \frac{u_i^2 - {u_i^{(0)}}^2}{2} &= \int_{v_i^{(0)}}^{v_i} \frac{\phi(z)}{\phi'(z)} dz, \quad\textrm{and}\quad
  u_i \phi(v_i) = \frac{1}{k},
\end{align*}
and $u_i = u_j$, $v_i = v_j$ for all $i, j \in [k]$.
\end{theorem}

The proof is analogous to that of Theorem \ref{thm: closed form 1 hidden layer}.  Sequence encoding can be viewed as a composition of individual networks $f_i$ that are trained to map example $x^{(i)}$ to example $x^{((i \mod n)+1)}$. The following theorem provides a sufficient condition for when the composition of these individual networks stores the sequence of training examples $\{x^{(i)}\}_{i=1}^{n}$ as a stable limit cycle.

\begin{theorem}
\label{thm: Sequence Encoding Eigenvalues}
Let $\{x^{(i)}\}_{i=1}^n$ be $n$ training examples with $\| x^{(i)} \|_2 = 1$  for all $i \in [n]$, and let $\{f_i\}_{i=1}^{n}$ denote $n$ 1-hidden layer networks satisfying the assumptions in Theorem \ref{thm: closed from sequence encoding} and trained on the loss in Eq.~\eqref{loss_sequence}.  Then the composition $f = f_n \circ f_{n-1} \circ \ldots \circ f_1$ satisfies:
\begin{align}
\label{product_term}
    \lambda_1(\mathbf{J}(f(x^{(1)}))= \prod\limits_{i=1}^{n}\left( \frac{\phi'(v_j^{(i)}) v_j^{(i)}}{\phi(v_j^{(i)})}  \right).
\end{align}
\end{theorem}

The proof is presented in Appendix \ref{appendix: E}. Theorem~\ref{thm: Sequence Encoding Eigenvalues} shows that sequence encoding provides a more efficient mechanism for memory than autoencoding.  If each of the networks $f_i$ autoencoded example $x_i$ for $i \in [n]$, then Theorem \ref{thm: Top Eigenvalue} implies that each of the $n$ training examples is an attractor (and thus recoverable) if \emph{each} term in the product in Eq.~\eqref{product_term} is less than $1$.  This in turn implies that the product, itself, is less than 1 and hence all training examples are stored by the corresponding sequence encoder, $f$, as a stable limit cycle.

\section{Discussion}

We have shown that standard over-parameterized neural networks trained using standard optimization methods implement associative memory.  In particular, we empirically showed that autoencoders store training examples as attractors and that sequence encoders store training sequences as stable limit cycles.  We then demonstrated that sequence encoders provide a more efficient mechanism for memorization and retrieval of data than autoencoders.  In addition, we mathematically proved that when trained on a single example, nonlinear fully connected autoencoders store the example as an attractor.  By modeling sequence encoders as a composition of maps, we showed that such encoders provide a more efficient mechanism for implementing memory than autoencoders, a finding which fits with our empirical evidence.  We end by discussing implications and possible future extensions of our results.  

\vspace{2mm}

\noindent \textbf{Inductive Biases.}  In the over-parameterized regime, neural networks can fit the training data exactly for different values of parameters.  In general, such interpolating auto-encoders do not store data as attractors (Corollary 1). Yet, as we showed in this paper, this is typically the case for parameter values chosen by  gradient-based optimization methods. Thus, our work identifies a novel \textit{inductive bias} of the specific solutions selected by the training procedure. 
Furthermore, increasing depth and width leads to networks becoming more \textit{contractive} around the training examples, as demonstrated in Figure \ref{fig: Figure 4}.  

While our paper concentrates on the question of implementing associative memory, we employ the same training procedures and similar network architectures to those used in standard supervised learning tasks.  We believe that our finding on the existence and ubiquity of attractors in these maps may shed light on the important question of inductive biases in interpolating neural networks for classification~\cite{belkin2019reconciling}.


\vspace{2mm}

\noindent \textbf{Generalization.}  While generalization in autoencoding often refers to the ability of a trained autoencoder to reconstruct test data with low error \cite{IdentityCrisis}, this notion of generalization may be problematic for the following reason.  The identity function achieves zero test error and thus ``generalizes'', although no training is required for implementing this function.  In general, it is unclear how to formalize generalization for autoencoding and alternate notions of generalization may better capture the  desired properties. An alternative definition of generalization is the ability of an autoencoder to map corrupted versions of training examples back to their originals (as in Fig.~\ref{fig: Figure 2}a,b). Under this definition, over-parameterized autoencoders storing training examples as attractors generalize (Fig.~\ref{fig: Figure 2}c), while the identity function does not generalize.  Given this issue with the current notion of generalization for autoencoding, it is an important line of future work to provide a definition of generalization that appropriately captures desired properties of trained autoencoders.  Lastly, another important direction of future work is to build on the properties of autoencoders and sequence encoders identified in this work to understand generalization properties of networks used for classification and regression.

\vspace{2mm}

\noindent \textbf{Metrics Used by Nonlinear Networks.} 
In Figure \ref{fig: Figure 3}, we provided a visualization of how the basins of attraction for individual training examples subdivide the space of inputs.  The picture appears very different from the Voronoi tessellation corresponding to the 1-nearest-neighbor (1-NN) predictor, where each input is associated to its closest training point in Euclidean distance. Yet, this may be different in high dimension.  In Appendix Fig~\ref{appendix: fig 14}, we compare the recovery rate of our network from Figure \ref{fig: Figure 2} to that of a 1-NN classifier and observe remarkable similarity, leading us to conjecture that the basins of attraction of high-dimensional fully connected neural networks may be closely related to the tessellations produced by 1-NN predictors.  Thus, understanding the geometry of attractors in high-dimensional neural networks is an important direction of future research.  
\vspace{2mm}

\noindent \textbf{Connection to Biological Systems.} Finally, another avenue for future exploration (and a key motivation for the original work on Hopfield networks~\cite{HopfieldNetwork}) is the connection of autoencoding and sequence encoding in neural nets to  memory mechanisms in  biological systems.  Since over-parameterized autoencoders and sequence encoders recover stored patterns via iteration, the retrieval mechanism presented here is biologically plausible.  However, back-propagation is not believed to be a biologically plausible mechanism for storing patterns ~\cite{BackpropagationImplausible}. An interesting avenue for future research is to 
identify storage mechanisms that are biologically plausible 
and to see whether similar attractor phenomena arise in other, more biologically plausible, optimization methods.   

\vspace{2mm}

\noindent \textbf{Materials and Methods.} An overview of all experimental details including datasets, network architectures, initialization schemes, random seeds, and training hyperparameters considered in this work are provided in Appendix Fig.~\ref{appendix: fig 6}, \ref{appendix: fig 15}, and \ref{appendix: fig 16}. Briefly, we used the PyTorch library \cite{PyTorch} and two NVIDIA Titan Xp GPUs for training all neural networks. In our autoencoding experiments on the image datasets ImageNet-64 \cite{PixelRNN}, CIFAR10 \cite{CIFAR10}, and MNIST \cite{mnist-lecun1998}, we trained both, fully connected networks and U-Net convolutional networks \cite{UNet}. For Fig~\ref{fig: Figure 3}, \ref{fig: Figure 4}b, \ref{fig: Figure 4}c, and \ref{fig: Figure 5} as well as for training sequence encoder models on audio and video samples\footnote{Link to video and audio samples: \url{https://github.com/uhlerlab/neural_networks_associative_memory}}, we used fully connected networks. For all these experiments we used the Adam optimizer with a learning rate of $10^{-4}$ until the mean squared error dropped below $10^{-8}$. 
For Appendix Fig.~\ref{appendix: fig 15} and \ref{appendix: fig 16}, we fixed the architecture width and depth while varying the initialization scheme, optimization method, and activation function.




\section*{Acknowledgements}

The authors thank the Simons Institute at UC Berkeley for hosting them during the summer 2019 program on ``Foundations of Deep Learning'' , which facilitated this work. A.~Radhakrishnan and C.~Uhler were partially supported by the National Science Foundation (DMS-1651995), Office of Naval Research (N00014-17-1-2147 and N00014-18-1-2765), IBM, and a Simons Investigator Award  to C.~Uhler. M.~Belkin acknowledges support from NSF (IIS-1815697 and IIS-1631460) and a Google Faculty Research Award.  The Titan Xp used for this research was donated by the NVIDIA Corporation.

\bibliographystyle{plain}
\bibliography{references}

\newpage
\appendix
\section{Encoding Multiple Sequences}
\label{appendix: A}
Given sequences of training examples $\{x_i^{(j)}\} \in \mathbb{R}^d$ for $i \in [n], j \in [k_i], k_i \in \mathbb{Z}_{\geq 0}$, minimizing the following \textit{sequence encoding} objective with gradient descent methods leads to training sequences being stored as limit cycles:
\begin{equation}
\label{eq: sequence encoder objective}
\arg\min_{f \in \mathcal{F}} \;\sum\limits_{i=1}^n \sum\limits_{j=1}^{k_i} \|f(x_i^{(j \hspace{-2mm} \mod k_i) + 1}) - x_i^{(j)}\|_2^2.
\end{equation}

\section{Proofs of Invariants, Theorem 1, and Theorem 2}
\label{appendix: B}
We present the full proof of Invariants 1 and 2, Theorem 1, and Theorem~2 below.  To simplify notation, we use $A$ to represent $W_1$ and $B$ to represent $W_2$.

\vspace{2mm}


Let $f(z) = A \phi(B z)$ represent a 1-hidden layer network with elementwise, differentiable nonlinearity $\phi$, $z \in \mathbb{R}^{k_0}$, $B \in \mathbb{R}^{k \times k_0}$, and $A \in \mathbb{R}^{k_0 \times k}$.  Suppose that gradient descent with learning rate $\gamma$ is used to minimize the following loss for the autoencoding problem with 1 training example $x$:
\begin{align}
    \mathcal{L}(x, f) = \frac{1}{2} \| x - f(x) \|_2^2.
\end{align}
Gradient descent updates on $A$ and $B$ are as follows:
\begin{align}
    A^{(t+1)} &= A^{(t)} + \gamma (x - A^{(t)} \phi(B^{(t)}x)) \phi(B^{(t)}x)^T \\
    B^{(t+1)} &= B^{(t)} + \gamma \diag(\phi'(B^{(t)} x)) {A^{(t)}}^T (x - A^{(t)} \phi(B^{(t)}x)) x^T,
\end{align}
where $\diag(\phi'(B^{(t)} x)) = \begin{bmatrix} \phi'(B_{1, :}^{(t)}x) & & \\
                                                    & \ddots & \\
                                                    & & \phi'(B_{k, :}^{(t)}x) \end{bmatrix}$ with $B_{i, :}^{(t)}$ representing row $i$ of matrix $B^{(t)}$.
\vspace{5mm}

In the following, we restate and prove Invariant 1 formally:
\begin{inv} 
    \label{prop: invariant}
    If $A^{(0)} = x {a^{(0)}}^T$ and $B^{(0)} = {b^{(0)}} x^T$ for $a^{(0)}, b^{(0)}  \in \mathbb{R}^{k}$, then for all time-steps $t$, $A^{(t)} = x {a^{(t)}}^T$ and $B^{(t)} = {b^{(t)}} x^T$ for $a^{(t)}, b^{(t)}  \in \mathbb{R}^{k}$.  
\end{inv}           

\begin{proof}
We provide a proof by induction.  The base case follows for $t = 0$ from the initialization.  Now we assume for some $t$ that $A^{(t)} = x {a^{(t)}}^T$ and $B^{(t)} = {b^{(t)}} x^T$.  Then for time $t+1$ we have:
\begin{align*}
    A^{(t+1)} &= A^{(t)} + \gamma (x - A^{(t)} \phi(B^{(t)}x)) \phi(B^{(t)}x)^T  \\
    &= x {a^{(t)}}^T + \gamma (x - x {a^{(t)}}^T \phi(B^{(t)}x)) \phi(B^{(t)}x)^T \\
    &= x [ {a^{(t)}}^T + \gamma (1 - {a^{(t)}}^T \phi(B^{(t)}x)) \phi(B^{(t)}x)^T ] \\
    &= x {a^{(t+1)}}^T
\end{align*}
and similarly,
\begin{align*}
    B^{(t+1)} &= B^{(t)} + \gamma \diag(\phi'(B^{(t)} x)) {A^{(t)}}^T (x - A^{(t)} \phi(B^{(t)}x)) x^T \\
    &= {b^{(t)}} x^T + \gamma \diag(\phi'(B^{(t)} x)) {A^{(t)}}^T (x - A^{(t)} \phi(B^{(t)}x)) x^T \\
    &= [{b^{(t)}} + \gamma \diag(\phi'(B^{(t)} x)) {A^{(t)}}^T (x - A^{(t)} \phi(B^{(t)}x))] x^T \\
    &= b^{(t+1)} x^T.
\end{align*}
Hence, since the statement holds for $t+1$, it holds for all time steps.  
\end{proof}

Under the initialization in Invariant \ref{prop: invariant}, outputs of the network are multiples of the training example.  Generalizing this result, in Materials and Methods \ref{sec: proof for outputs in span of training data}, we prove that autoencoders trained on multiple examples produce outputs in the span of the training data.  

Using Invariant \ref{prop: invariant}, any interpolating solution satisfies the following condition.   

\begin{prop}
\label{cor: Corollary 1}
Under the initialization in Invariant \ref{prop: invariant}, if $\|x\|_2 = 1$ and $A^{(\infty)}, B^{(\infty)}$ yield zero training error, then $A^{(\infty)} = x a^T, B^{(\infty)} = b x^T$ for $a, b \in \mathbb{R}^k$ such that $a^T \phi(b) = 1$.   
\end{prop}

\begin{proof}
From Invariant \ref{prop: invariant}, it holds that $A^{(\infty)} = xa^T, B^{(\infty)} = bx^T$.  If the loss is minimized to zero, then $f(x) = x$, and thus:
\begin{align*}
    x = xa^T \phi(bx^T x) = xa^T \phi(b) \rightarrow a^T \phi(b) = 1,
\end{align*}
which completes the proof.
\end{proof}

Using Proposition \ref{cor: Corollary 1}, we can compute the maximum eigenvalue of the Jacobian at training example $x$ (denoted $\lambda_1(\mathbf{J}(f(x)))$. This is done in the following result.

\begin{prop}
\label{cor: Corollary 2}
Let the Hadamard product $\odot$ denote coordinate-wise multiplication. Then $\lambda_1(\textbf{J}(f(x))) = a^T \phi'(b) \odot b$. 
\end{prop}

\begin{proof} Note that
\begin{align*}
    \mathbf{J}(f(z)) &= x a^T (\phi'(b x^T z) \odot b) x^T  = x a^T(\phi'(b) \odot b) x^T,  
\end{align*}
and hence $\mathbf{J}(f(x)) x = x a^T(\phi'(b) \odot b)$. This implies that $x$ is an eigenvector of $\mathbf{J}(f(x))$ with eigenvalue $a^T(\phi'(b) \odot b)$. Since $x a^T(\phi'(b) \odot b) x^T$ is rank $1$, the remaining eigenvalues of $\mathbf{J}(f(x))$ are all zero, which completes the proof.
\end{proof}

We now prove Invariant \ref{prop:Equal Init}, which states that if, in addition to Invariant 1, all weights in each row of $A, B$ are initialized to be equal, then they remain equal throughout training.

\begin{inv}
\label{prop:Equal Init}
Let $\|x\|_2 = 1$ and $A^{(0)} = x {a^{(0)}}^T, B^{(0)} = b^{(0)} x^T $ for vectors $a^{(0)}, b^{(0)} \in \mathbb{R}^k$. If $a_i^{(0)} = a_1^{(0)}, b_i^{(0)} = b_1^{(0)}$, then $a_i^{(t)} = a_1^{(t)}, b_i^{(t)} = b_1^{(t)}$ for all $i \in [k]$.  
\end{inv}
\begin{proof}
From the proof of Invariant \ref{prop: invariant}, we have that:
\begin{align*}
{a^{(t+1)}}^T &= {a^{(t)}}^T + \gamma (1 - {a^{(t)}}^T \phi(b^{(t)})) \phi(b^{(t)})^T  \\
b^{(t+1)} &= {b^{(t)}} + \gamma \diag(\phi'(b^{(t)})) {a^{(t)}}^T (1 - {a^{(t)}}^T \phi(b^{(t)})).
\end{align*}
Now, if  $a_i^{(t)} = a_1^{(t)}, b_i^{(t)} = b_1^{(t)}$, then from the above this implies that $a_i^{(t+1)} = a_1^{(t+1)}, b_i^{(t+1)} = b_1^{(t+1)}$ and the proof follows by induction.  
\end{proof}

Using Invariants \ref{prop: invariant} and \ref{prop:Equal Init}, we can now prove Thoerem 1. 

\begin{theorem}
\label{thm: closed form 1 hidden layer}
Let $f(z) = A \phi(B z)$ denote a 1-hidden layer network with elementwise nonlinearity $\phi$, $z \in \mathbb{R}^{k_0}$, $B \in \mathbb{R}^{k \times k_0}$, and $A \in \mathbb{R}^{k_0 \times k}$.  Let $A, B$ be initialized as in Invariant 1, 2.  Gradient descent with learning rate $\gamma$ is used to to minimize the following loss for $1$ training example $x \in \mathbb{R}^{k_0}$ with $\|x\|_2 = 1$ 
\begin{align}
    \mathcal{L}(x, f) = \frac{1}{2} \| x - f(x) \|_2^2.
\end{align}
Assuming $\frac{\phi(z)}{\phi'(z)} < \infty ~\forall z \in \mathbb{R}$, then as the learning rate $\gamma \rightarrow 0$, it holds that $A^{(\infty)} = xa^T$ and $B^{(\infty)} = bx^T$ with $a, b \in \mathbb{R}^{k}$ such that
\begin{align*}
  \frac{a_i^2 - {a_i^{(0)}}^2}{2} &= \int_{b_i^{(0)}}^{b_i} \frac{\phi(z)}{\phi'(z)} dz ~~;~~
  a_i \phi(b_i) = \frac{1}{k}
\end{align*}
for all $i \in [k]$ with $a_i = a_j, b_i = b_j$ for all $i, j \in [k]$.  
\end{theorem}

\begin{proof}
From Invariants \ref{prop: invariant} and \ref{prop:Equal Init}, we have that:
\begin{align*}
a_1^{(t+1)} &= a_1^{(t)} + \gamma (1 - k a_1^{(t)} \phi(b_1^{(t)})) \phi(b_1^{(t)})  \\
b_1^{(t+1)} &= {b_1^{(t)}} + \gamma \phi'(b_1^{(t)}) a_1^{(t)} (1 - k a_1^{(t)} \phi(b_1^{(t)})).
\end{align*}
Rearranging the above, we obtain
\begin{align*}
    \frac{a_1^{(t+1)} - a_1^{(t)}}{\phi(b_1^{(t)})} &= \frac{b_1^{(t+1)} - {b_1^{(t)}}}{a_1^{(t)} \phi'(b_1^{(t)})} \\
    \implies a_1^{(t)} \frac{da_1^{(t)}}{dt}&= \frac{\phi(b_1^{(t)})}{\phi'(b_1^{(t)})} \frac{db_1^{(t)}}{dt} ~~~~ \text{as $\gamma \rightarrow 0$}\\
    \implies \int\limits_0^{t'} a_1^{(t)} \frac{da_1^{(t)}}{dt} dt&= \int\limits_0^{t'} \frac{\phi(b_1^{(t)})}{\phi'(b_1^{(t)})} \frac{db_1^{(t)}}{dt} dt \\
    \implies \frac{{a_1^{(t')}}^2 - {a_1^{(0)}}^2}{2} &= \int\limits_{b_1^{(0)}}^{b_1^{(t')}}  \frac{\phi(z)}{\phi'(z)} dz,
\end{align*}
wich completes the proof.
\end{proof}

Using Proposition \ref{cor: Corollary 2} and Theorem \ref{thm: closed form 1 hidden layer} above, we can calculate the values of $a_1, b_1$ explicitly.  After computing $b_1$, the following result can be used to compute $\lambda_1(\mathbf{J}(f(x)))$.

\begin{theorem}
\label{thm: 1 Hidden Layer Eigenvalue}
Under the setting of Theorem \ref{thm: closed form 1 hidden layer} it holds that
\begin{align*}
\lambda_1(\textbf{J}(f(x))) = \frac{\phi'(b_1)b_1}{\phi(b_1)}.
\end{align*} 
\end{theorem}

\begin{example}
From Theorem \ref{thm: closed form 1 hidden layer}, if $\phi(z) = e^{2z}$, then the values of $a_1, b_1$ are given by solving:
\begin{align*}
    a_1^2 &= b_1, \\
    \sqrt{b_1} e^{2 b_1} &= \frac{1}{k}.
\end{align*}
From Theorem \ref{thm: 1 Hidden Layer Eigenvalue}, the top eigenvalue of the Jacobian at the training example is $2b_1$.  Note that the above equation implies $b_1$ decreases as $k$ increases. At $k = 1, \lambda_1 = .6$, at $k=2, \lambda_1 = .28$, and at $k=3, \lambda_1 = .16$.  This nonlinearity guarantees that training example is an attractor.  Moreover, as there are no fixed points other than the training example, there are no spurious attractors.  
\end{example}

\textbf{Remark.}  The analysis of Theorem \ref{thm: 1 Hidden Layer Eigenvalue} implies that attractors arise as a consequence of training and are not simply consequences of interpolation by a neural network with a certain architecture; see the following corollary. 

\begin{corollary}  Let $x \in \mathbb{R}^{k_0}$ with $\|x\|_2 = 1$ and  $f(z) = xa^T \phi(bx^T z)$, where $a, b \in \mathbb{R}^k$ and $\phi$ is a smooth element-wise nonlinearity with $\frac{\phi'(z)}{\phi(z)} < \infty$ for all $z \in \mathbb{R}$, $\left|\frac{\phi'(z) z}{\phi(z)}\right| > 1$ for $z$ in an open interval $\mathcal{O} \subset \mathbb{R}$.  Then there exist infinitely many $b \in \mathbb{R}^k$, such that $f(x) = x$ and $x$ is not an attractor for $f$.
\end{corollary}

\begin{proof}
If $a_i = a_j$ and $b_i = b_j$ for all $i, j \in  [k]$, then Propositions \ref{cor: Corollary 1} and \ref{cor: Corollary 2} imply that $f(x) = x$ if:
\begin{align*}
    a_i \phi(b_i) &= \frac{1}{k} \\
    \lambda_1(\mathbf{J}(f(x))) &= \frac{\phi'(b_i) b_i}{\phi(b_i)}.
\end{align*}
However, for any value of $b_i$ such that $\phi(b_i) \neq 0$, we can select a value of $a_i$ such that $a_i \phi(b_i) = \frac{1}{k}$.  Hence, we just select appropriate $a_i$ such that $a_i \phi(b_i) = \frac{1}{k}$ for $b_i \in \mathcal{O}$.
\end{proof}

\section{Analysis of Deep Autoencoders with 1 Training Example}
\label{appendix: C}
Let $f(z) = W_d \phi(W_{d-1} \phi(W_{d-2} \ldots \phi(W_1 z) \ldots ) )$ represent a $d-1$ hidden layer network with elementwise nonlinearity $\phi$ with $z \in \mathbb{R}^{k_0}, W_i \in \mathbb{R}^{k_i \times k_{i-1}}$, and $k_d = k_0$.  We again consider the setting where gradient descent with learning rate $\gamma$ is used to minimize the square loss on 1 training example $x$.  As in the 1 hidden layer case, we derive invariants of training that allow us to derive a closed form solution when training on 1 example in the gradient flow setting.  

Firstly, the following invariant (analogous to Invariant \ref{prop: invariant}) holds in the deep setting.  

\begin{inv}
\label{prop: Invariant Deep}
If $W_d^{(0)} = x{a^{(0)}}^T$ and $W_1^{(0)} = b^{(0)}x^T$ for $b^{(0)}, a^{(0)} \in \mathbb{R}^{k_1}, \mathbb{R}^{k_{d-1}}$ respectively, then for all time-steps t, $W_d^{(t)} = x{a^{(t)}}^T$ and $W_1^{(t)} = b^{(t)}x^T$ for $b^{(t)}, a^{(t)} \in \mathbb{R}^{k_1}, \mathbb{R}^{k_{d-1}}$ respectively. 
\end{inv}

The proof is analogous to that of Invariant \ref{prop: invariant}.  Now that we have invariants of training for layers $W_d, W_1$, we still need to find an invariant for the intermediate layers $W_2, \ldots, W_{d-1}$.  The following invariant extends Invariant \ref{prop:Equal Init} to the deep setting.

\begin{inv}
\label{prop: Equal Init Deep}
Assume $\| x\|_2 = 1$ and $W_d^{(0)} = x{a^{(0)}}^T$, $W_1^{(0)} = b^{(0)}x^T$ for $b^{(0)}, a^{(0)} \in \mathbb{R}^{k_1}, \mathbb{R}^{k_{d-1}}$ respectively.  Let $\mathbf{1}_{m \times n}$ denote the $m \times n$ matrix of all $1's$.  If $a_j^{(0)} = a_1^{(0)}, b_l^{(0)} = b_1^{(0)}, W_i^{(0)} = w_i^{(0)} \mathbf{1}_{k_i \times k_{i-1}}$ with $w_i^{(0)} \in \mathbb{R}$, then for all time-steps $t$, $a_j^{(t)} = a_1^{(t)}, b_l^{(t)} = b_1^{(t)}, W_i^{(t)} = w_i^{(t)} \mathbf{1}_{k_i \times k_{i-1}}$ for $j \in \mathbb{R}^{k_{d-1}}$, $l \in \mathbb{R}^{k_1}$, and $i \in \{2, \ldots, d-1\}$.   
\end{inv}

That is, in the deep setting, the intermediate layers remain rank 1 throughout training, if they are initialized to be a constant times the all 1's matrix.  The proof follows by induction and is analogous to the proof of Invariant \ref{prop: Equal Init Deep}.  

\begin{theorem*}
Let $f(z) = W_d \phi(W_{d-1} \phi(W_{d-2} \ldots \phi(W_1 z) \ldots ) )$ denote a $d-1$ hidden layer network with elementwise nonlinearity $\phi$, with $z \in \mathbb{R}^{k_0}, W_i \in \mathbb{R}^{k_i \times k_{i-1}}$, and $k_d = k_0$.  Let $\{W_i\}_{i=1}^{d}$ be initialized as in Invariants 3, 4.  Gradient descent with learning rate $\gamma$ is used to to minimize the following loss for $1$ training example $x \in \mathbb{R}^{k_0}$ with $\|x\|_2 = 1$:

\begin{align}
    \mathcal{L}(x, f) = \frac{1}{2} \| x - f(x) \|_2^2
\end{align}

Assuming $\frac{\phi(z)}{\phi'(z)} < \infty ~\forall z \in \mathbb{R}$, then as the learning rate $\gamma \rightarrow 0$, we have the following relationships between the weights $a_1^{(t)}, b_1^{(t)}, w_i^{(t)} \in \mathbb{R}$ for $i \in\{1, \ldots d-2\}$:
\begin{align*}
    \frac{{w_1^{(t')}}^2 - {w_1^{(0)}}^2}{2} &= \int\limits_{b_1^{(0)}}^{b_1^{(t')}} \frac{\phi(b_1)}{\phi'(b_1) k_2} db_1, \\
    \frac{{w_{i+1}^{(t)}}^2 - {w_{i+1}^{(0)}}^2}{2} &= \int\limits_{w_i^{(0)}}^{w_{i}^{(t)}} \frac{\phi(w_i k_i \phi(w_{i-1} k_{i-1} \ldots \phi(b_1) \ldots))}{\phi'(w_i k_i \phi(w_{i-1} k_{i-1} \ldots \phi(b_1) \ldots))}  \frac{1}{\phi(w_{i-1} k_{i-1} \ldots \phi(b_1) \ldots) k_{i+1}} dw_i, \\
    \frac{{a_1^{(t)}}^2 - {a_1^{(0)}}^2}{2} &= \int\limits_{w_{d-1}^{(0)}}^{w_{d-2}^{(t)}} \frac{\phi(w_{d-1} k_{d-1} \phi(w_{d-2} k_{d-2} \ldots \phi(b_1) \ldots))}{\phi'(w_{d-1} k_{d-1} \phi(w_{d-2} k_{d-2} \ldots \phi(b_1) \ldots))}  \frac{1}{\phi(w_{d-2} k_{d-2} \ldots \phi(b_1) \ldots)} dw_{d-1}.
\end{align*}
\end{theorem*}

The proof is analogous to the proof of Theorem \ref{thm: closed form 1 hidden layer}.  In addition, analogously to Theorem \ref{thm: 1 Hidden Layer Eigenvalue}, we can explicitly compute the maximum eigenvalue of the Jacobian at the training example $x$ using the following corollary.

\begin{corollary}
Under the setting of the theorem above, it holds that
\begin{align*}
    \lambda_1(\mathbf{J}(f(x))) = \frac{\left(\prod_{i=2}^{d-1}\phi'(w_{i} k_{i-1} \phi(w_{i-1} k_{i-2}\ldots \phi(b_1) \ldots))w_i k_{i-1}\right) \phi'(b_1) b_1}{\phi(w_{d-1} k_{d-2} \phi(w_{d-2} k_{d-3} \ldots (\phi(b_1)) \ldots))}
\end{align*}
\end{corollary}

The proof is analogous to that of Theorem \ref{thm: 1 Hidden Layer Eigenvalue}.  Note that in the theorem above, the integration for later layer weights becomes increasingly complicated, since the integration depends on the values of the previous weights.  Fortunately, for certain nonlinearities, the integral is tractable; see the example below.  

\begin{example}
Let $\phi(x) = x^m$ and assume $k_i = 1$ for $i \in [d-1]$.  Then from the theorem above, it holds that:
\begin{align*}
    w_{i+1}^2 = \frac{w_i^2}{m} ~` \forall i \in [d-1]
\end{align*}
From the corollary above, $\lambda_1(\mathbf{J}(f(x))) = m^{d-1}$.  Hence if $m < 1$, then $x$ becomes an attractor.  
\end{example}

\textbf{Remarks.} In the 1 hidden layer setting, regardless of how large the width, $\lambda_1(\mathbf{J}(f(x))) = m$ when using $\phi(z) = z^m$.  Thus, the example above demonstrates that depth can make over-parameterized autoencoders more contractive even when width cannot.

\section{Trained Autoencoders Produce Outputs in the Span of the \\Training Data}
\label{appendix: D}
In the following section, we generalize Invariant 1 to the setting with multiple training examples and thus, demonstrate that trained autoencoders produce outputs in the span of the training data.

\label{sec: proof for outputs in span of training data}

\begin{inv}
    Let $f(z) = A \phi(B z)$ represent a 1-hidden layer network with elementwise, differentiable nonlinearity $\phi$, $z \in \mathbb{R}^{k_0}$, $B \in \mathbb{R}^{k \times k_0}$, and $A \in \mathbb{R}^{k_0 \times k}$.  Suppose that gradient descent with learning rate $\gamma$ is used to minimize the following loss for the autoencoding problem with $n$ training examples $\{x^{(i)}\}_{i=1}^n$:
\begin{align}
    \mathcal{L}(x, f) = \frac{1}{2} \| x^{(i)} - f(x^{(i)}) \|_2^2.
\end{align}
    If $A^{(0)} = \sum\limits_{i=1}^{n}x^{(i)} {a_i^{(0)}}^T$ and $B^{(0)} = \sum\limits_{i=1}^{n}{b_i^{(0)}} {x^{(i)}}^T$ for vectors $a_i^{(0)}, b_i^{(0)}  \in \mathbb{R}^{k}$, then for all time-steps $t$, it holds that $A^{(t)} = \sum\limits_{i=1}^{n}x^{(i)} {a_i^{(t)}}^T$ and $B^{(t)} = \sum\limits_{i=1}^{n}{b_i^{(t)}} {x^{(i)}}^T$ for some vectors $a_i^{(t)}, b_i^{(t)}  \in \mathbb{R}^{k}$.  
\end{inv}

\begin{proof}
The proof exactly follows the proof of Invariant \ref{prop: invariant}.  For completeness, we show the proof for $A^{(t)}$ below.  We again provide a proof by induction.  The base case follows for $t = 0$ from the initialization.  Now we assume for some $t$ that $A^{(t)} = \sum\limits_{i=1}^{n}x^{(i)} {a_i^{(t)}}^T$ and $B^{(t)} = \sum\limits_{i=1}^{n}{b_i^{(t)}} {x^{(i)}}^T$.  Then for time $t+1$ we have:
\begin{align*}
    A^{(t+1)} &= A^{(t)} + \gamma \sum\limits_{i=1}^{n}(x^{(i)} - A^{(t)} \phi(B^{(t)}x^{(i)})) \phi(B^{(t)}x^{(i)})^T  \\
    &= \sum\limits_{i=1}^{n}x^{(i)} {a_i^{(t)}}^T + \gamma \sum\limits_{i=1}^{n}(x^{(i)} - \sum\limits_{j=1}^{n}x^{(j)} {a_j^{(t)}}^T \phi(B^{(t)}x^{(i)})) \phi(B^{(t)}x^{(i)})^T \\
    &= \sum\limits_{i=1}^{n} x^{(i)} [ {a_i^{(t)}}^T + \gamma (1 - \sum\limits_{j=1}^{n}{a_i^{(t)}}^T \phi(B^{(t)}x^{(j)})) \phi(B^{(t)}x^{(j)})^T ] \\
    &=  \sum\limits_{i=1}^{n} x^{(i)} {a_i^{(t+1)}}^T.
\end{align*}
The proof for $B^{(t)}$ follows analogously.  Hence, since the statement holds for $t+1$, it holds for all time steps, which completes the proof. 
\end{proof}

\section{Proofs of Theorem 3 and 4}
\label{appendix: E}
By analyzing sequence encoders as a composition of maps, we prove that sequence encoding provides a more efficient mechanism for memory than autoencoding.  We begin by restating Theorem \ref{thm: closed from sequence encoding} below.

\begin{theorem}
\label{thm: closed from sequence encoding}
Let $f(z) = W_1 \phi(W_2 z)$ denote a 1-hidden layer network with elementwise nonlinearity $\phi$ and weights $W_1 \in \mathbb{R}^{k_0 \times k}$ and $W_2 \in \mathbb{R}^{k \times k_0}$, applied to $z \in \mathbb{R}^{k_0}$. Let $x^{(i)}, x^{(i+1)} \in \mathbb{R}^{k_0}$ be training examples with $\|x^{(i)}\|_2 = \|x^{(i+1)}\|_2 = 1$. Assuming that $\frac{\phi(z)}{\phi'(z)} < \infty ~\forall z \in \mathbb{R}$ and there exist $u^{(0)}, v^{(0)} \in \mathbb{R}^{k}$ such that $W_1^{(0)} = x^{(i+1)}{u^{(0)}}^T$ and $W_2^{(0)} = v^{(0)}{x^{(i)}}^T$ with $u_i^{(0)} = u_j^{(0)}, v_i^{(0)} = v_j^{(0)} ~\forall i, j \in [k]$, then
gradient descent with learning rate $\gamma\to 0$ applied to minimize
\begin{align}
    \mathcal{L}(x, f) = \frac{1}{2} \| x^{(i+1)} - f(x^{(i)}) \|_2^2
\end{align}
leads to a rank 1 solution
$W_1^{(\infty)} = x^{(i+1)} u^T$ and $W_2^{(\infty)} = v {x^{(i)}}^T$ with $u, v \in \mathbb{R}^{k}$ satisfying
\begin{align*}
  \frac{u_i^2 - {u_i^{(0)}}^2}{2} &= \int_{v_i^{(0)}}^{v_i} \frac{\phi(z)}{\phi'(z)} dz, \quad\textrm{and}\quad
  u_i \phi(v_i) = \frac{1}{k},
\end{align*}
and $u_i = u_j$, $v_i = v_j$ for all $i, j \in [k]$.
\end{theorem}

The proof exactly follows that of Theorem \ref{thm: closed form 1 hidden layer}, since updates to $u, v$ do not depend on the data $x^{(i)}, x^{(i+1)}$.      

\begin{theorem}
\label{thm: Sequence Encoding Eigenvalues}
Let $\{x^{(i)}\}_{i=1}^n$ be $n$ training examples with $\| x^{(i)} \|_2 = 1$  for all $i \in [n]$, and let $\{f_i\}_{i=1}^{n}$ denote $n$ 1-hidden layer networks satisfying the assumptions in Theorem \ref{thm: closed from sequence encoding} and trained on the loss in Eq.~\eqref{loss_sequence}.  Then the composition $f = f_n \circ f_{n-1} \circ \ldots \circ f_1$ satisfies:
\begin{align}
    \lambda_1(\mathbf{J}(f(x^{(1)}))= \prod\limits_{i=1}^{n}\left( \frac{\phi'(v_j^{(i)}) v_j^{(i)}}{\phi(v_j^{(i)})}  \right).
\end{align}
\end{theorem}

\begin{proof}
From Theorem \ref{thm: closed from sequence encoding}, each of the $f_i$ for $i \in [n]$ have the following form after training:
\begin{align*}
    f_i(z) = x^{(i\hspace{-2mm}\mod n + 1)} {u^{(i)}}^T \phi(v^{(i)} {x^{(i)}}^T z)
\end{align*}
with ${u^{(i)}}^T \phi(v^{(i)}) = 1$ and $u_l^{(i)} = u_j^{(i)}, v_l^{(i)} = v_j^{(i)}$ for $l, j \in [k]$ and $i \in [n]$.

Hence the composition $f$ is given by:
\begin{align*}
    f(z) &= (f_n \circ f_{n-1} \circ \ldots \circ f_1)(z) \\
    &= x^{(1)} {u^{(n)}}^T \phi(v^{(n)} {x^{(n)}}^T x^{(n)} {u^{(n-1)}}^T \phi(v^{(n-1)} x^{(n-1)} \ldots \phi(v^{(1)} {x^{(1)}}^T z) \ldots)) \\
    &= x^{(1)} {u^{(n)}}^T \phi(v^{(n)}  {u^{(n-1)}}^T \phi(v^{(n-1)} \ldots \phi(v^{(1)} {x^{(1)}}^T  z) \ldots)) ~~~~ \text{since $\|x^{(i)}\|_2 = 1$ for all $i \in [n]$.}
\end{align*}

We now compute the Jacobian of $f$ at the example $x^{(1)}$.  Since $u_i k \phi(v_i) = 1$ for all $i \in [n]$ and $\|x^{(1)}\|_2 = 1$, it holds that
\begin{align*}
    \mathbf{J}(f(x^{(1)})) = x^{(1)} \left(\prod_{i=1}^{n} u_j^{(i)} k \phi'(v_j^{(i)}) v_j^{(i)} \right) {x^{(1)}}^T.
\end{align*}

However, we know that $u_j^{(i)} = \frac{1}{k \phi(v_j^{(i)})}$, and so we have that
\begin{align*}
    \lambda_1(\mathbf{J}(f(x^{(1)}))= \prod\limits_{i=1}^{n}\left( \frac{\phi'(v_j^{(i)}) v_j^{(i)}}{\phi(v_j^{(i)})}  \right),
\end{align*}
which completes the proof.
\end{proof}

\section{Limit Cycles in Recurrent Neural Networks}
\label{appendix: F}
In order to demonstrate that recurrent neural networks (RNNs) can also memorize and recall sequences, we trained a vanilla RNN (whose architecture is detailed in Appendix Figure \ref{appendix: fig 6}) to encode the following sentence from our introduction: ``Hopfield networks are able to store binary training patterns as attractive fixed points.''  When training the RNN, we encoded each word using 1-hot representation i.e., since there are 13 words in the sentence, we represented each word with a vector of size 13 and placed a ``1'' in the index corresponding to the word.  

We trained such that each word is mapped to the next modulo 13 using the Cross Entropy Loss (as is done in practice).  Unlike the other settings considered in this work, RNNs are used to generate new sentences after training by sampling a new word from the vector output given a previous word\footnote{This process is usually started from inputting the all zero vector.}.  Under our architecture, we found that repeatedly choosing the highest probability word given the previous word consistently output the entire training sentence regardless of the number of times this sampling process was repeated.

\section{Sequence Encoding Audio}
\label{appendix: G}
In order to demonstrate that fully connected networks can memorize high dimensional sequences, we captured an audio clip of an 8 second recording from the Donald Trump talking pen.  Each second of the audio contains $22,050$ frequencies, and we trained a fully connected network to map from the frequencies in second $i$ to second $i+1 \mod 8$.  We have attached an audio sample\footnote{Located at: \url{https://github.com/uhlerlab/neural_networks_associative_memory}} titled ``trump\_quote\_recovered\_from\_noise.mp3'' demonstrating that iteration from random noise leads to recovery of the entire quote.  The full architecture used is provided in Appendix Figure \ref{appendix: fig 6}.

\begin{figure}[!ht]
\centering
\includegraphics[height=5.in]{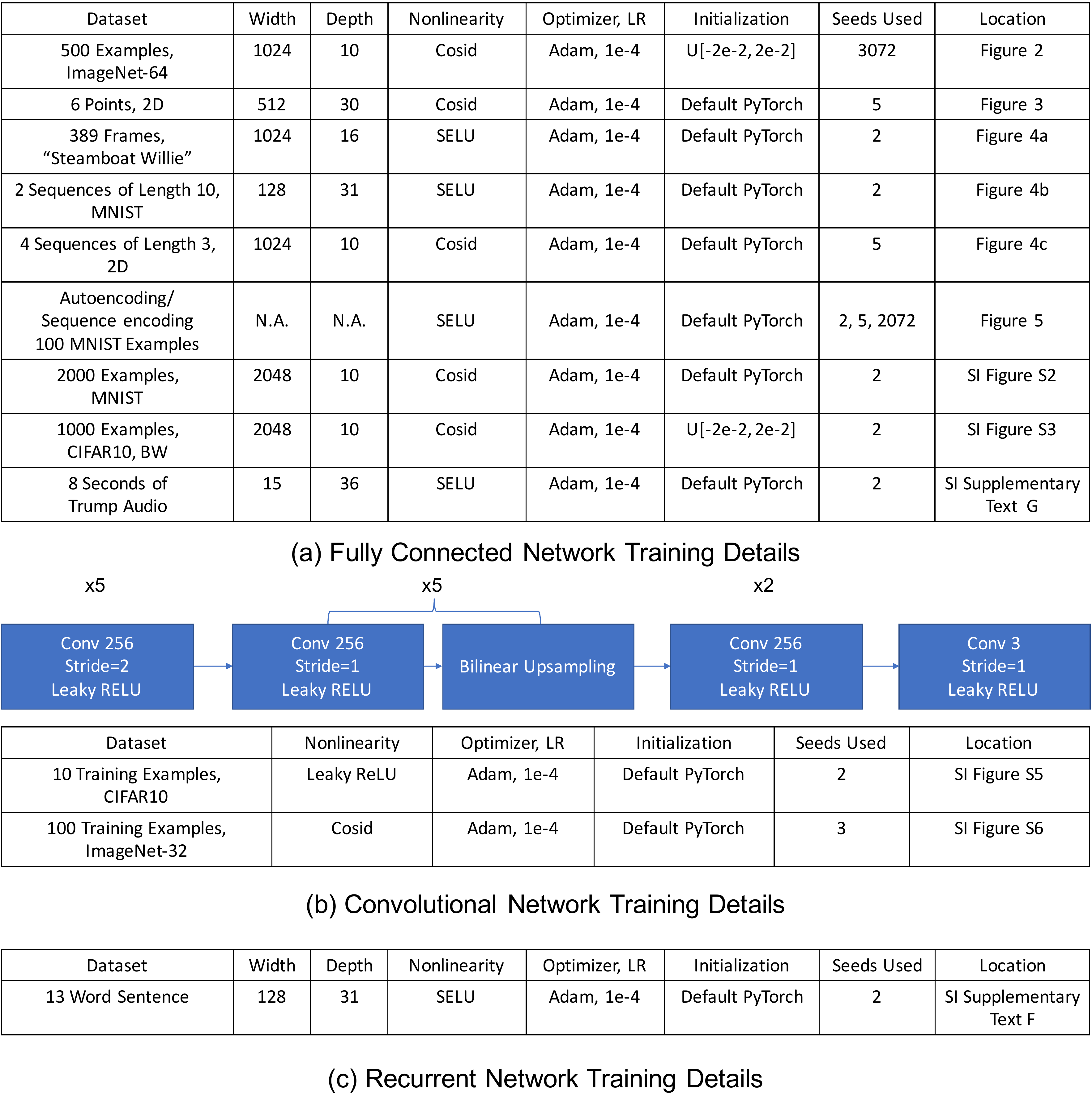}
\caption{Training details for all experiments in main text and Appendix.  Unless otherwise stated, all fully connected and convolutional networks are trained to minimize mean squared error below $10^{-8}$.  (a) Training details for fully connected architectures including dataset description, network width, network depth, nonlinearity, optimization method, learning rate, initialization scheme, random seeds used, and reference to the experiment in the text. (b) Training details for convolutional architecture including network topology, dataset, nonlinearity, optimization method, learning rate, initialization scheme, random seed and reference to experiment in the text.  (c) Training details for recurrent architecture including network width and depth (for producing the next hidden state and output state), nonlinearity, optimization method, learning rate, initialization scheme, random seed used, and reference to experiment in the text.} 
\label{appendix: fig 6}
\end{figure}

\begin{figure}
\centering
\includegraphics[height=1.5in]{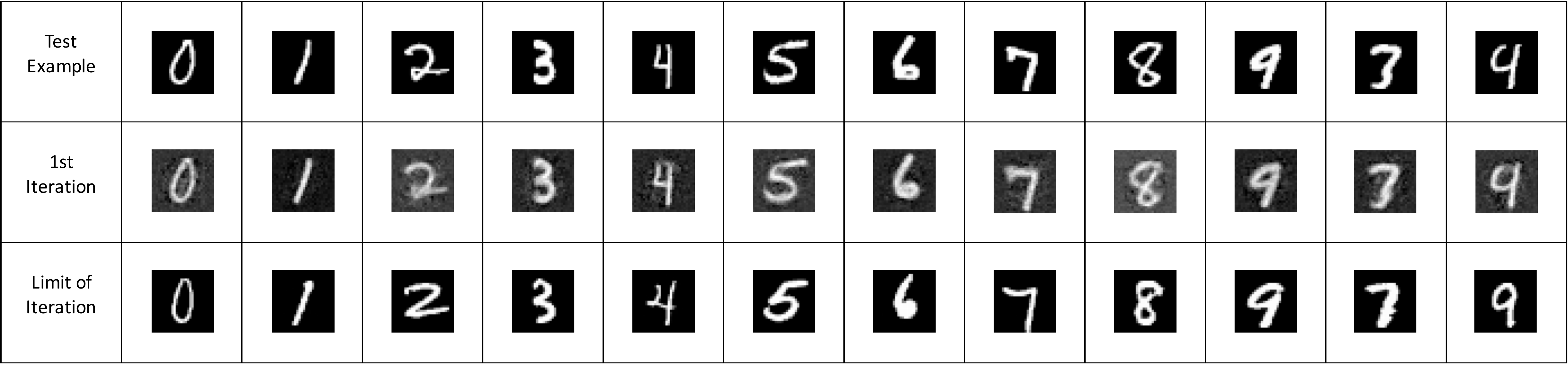}
\caption{Network trained on 2000 examples from MNIST stores all training examples as attractors.  1 iteration of the network leads to good reconstruction of the test example, but taking the limit of iteration leads to recovery of training examples.  Average reconstruction error after 1 iteration of 58000 test examples is $0.0135$. Training details are provided in Appendix Figure \ref{appendix: fig 6}.} 
\label{appendix: fig 7}
\end{figure}

\begin{figure}
\centering
\includegraphics[height=4in]{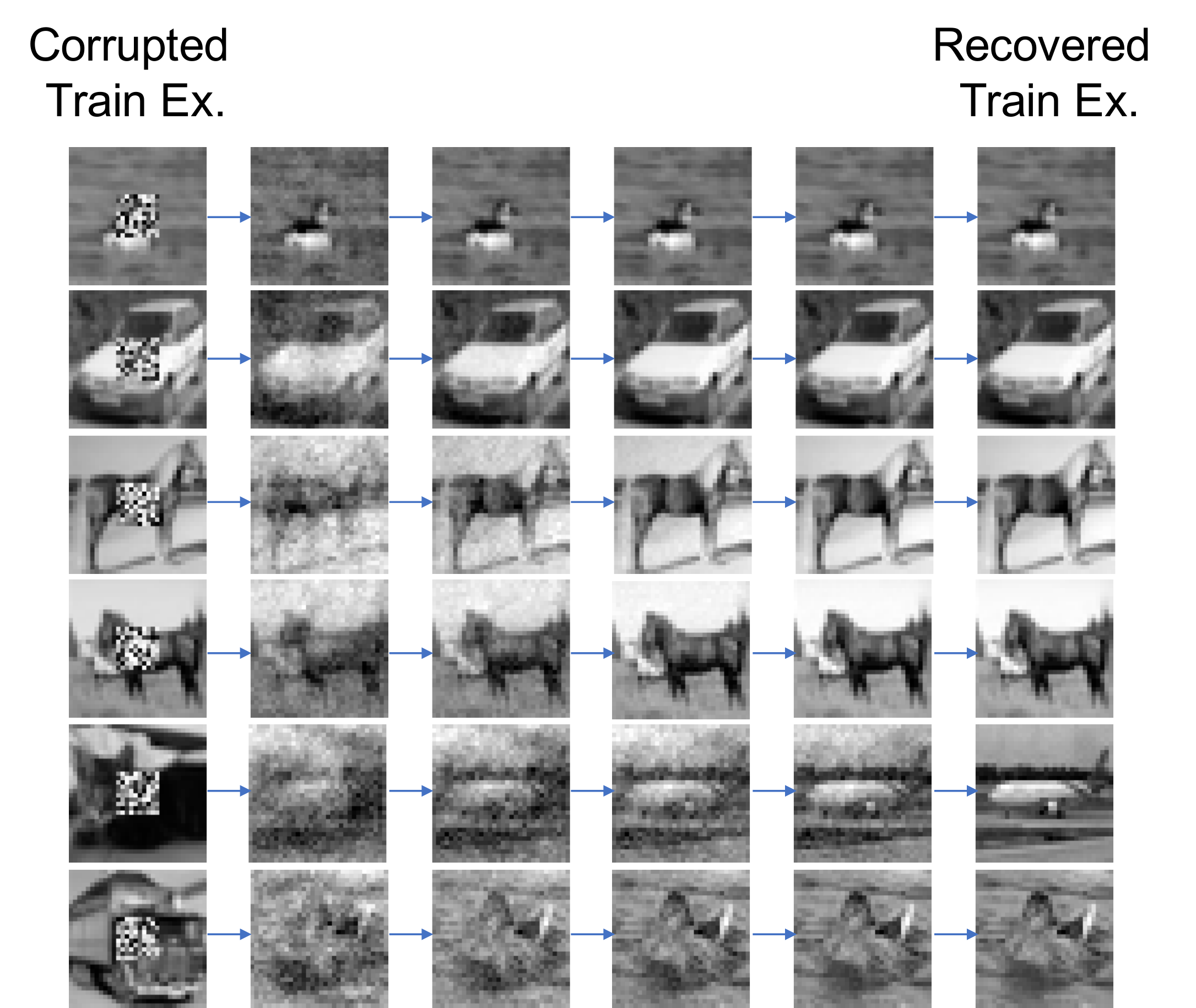}
\label{fig: CIFAR10 BW}
\caption{Network trained on 1000 black and white images from CIFAR10 stores all training examples as attractors.  Analogously to Figure \ref{fig: Figure 2}a of the main text, as training examples are attractors, iteration from corrupted inputs results in the recovery of a training example.  Training details are provided in Appendix Figure \ref{appendix: fig 6}.} 
\label{appendix: fig 8}
\end{figure}

\begin{figure}
\centering
\includegraphics[height=2.5in]{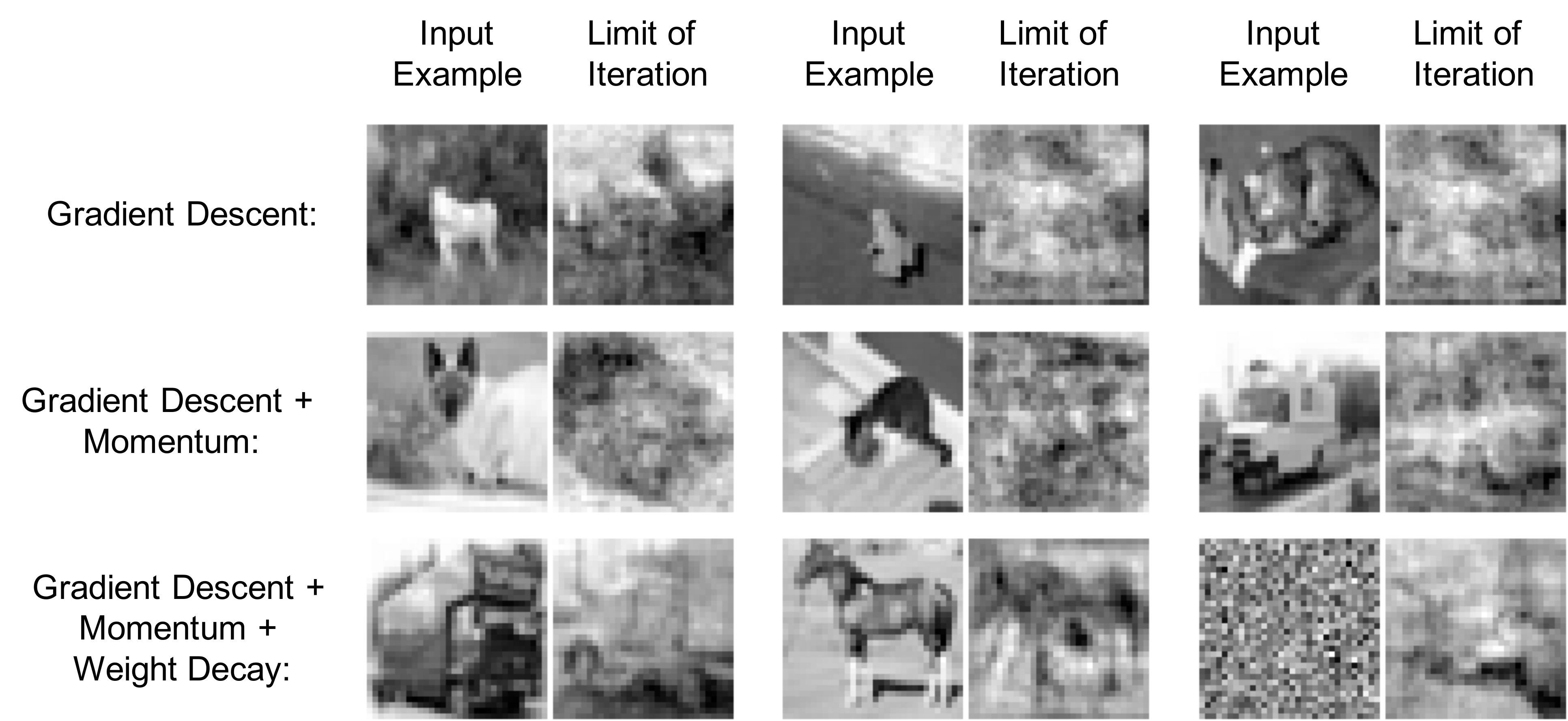}
\label{fig: Spurious Attractors}
\caption{Examples of spurious training examples arising in over-parameterized autoencoders trained using different optimization methods on 100 black and white examples from CIFAR10.  The networks used have 11 hidden layers, 256 hidden units per layer, SELU nonlinearity, and are initialized using the default PyTorch initialization scheme.  For all optimization methods, we use a learning rate of $10^{-1}$.  We use a momentum value of $0.009$ and weight decay of $0.0001$.} 
\label{appendix: fig 9}
\end{figure}

\begin{figure}
\centering
\includegraphics[height=3.in]{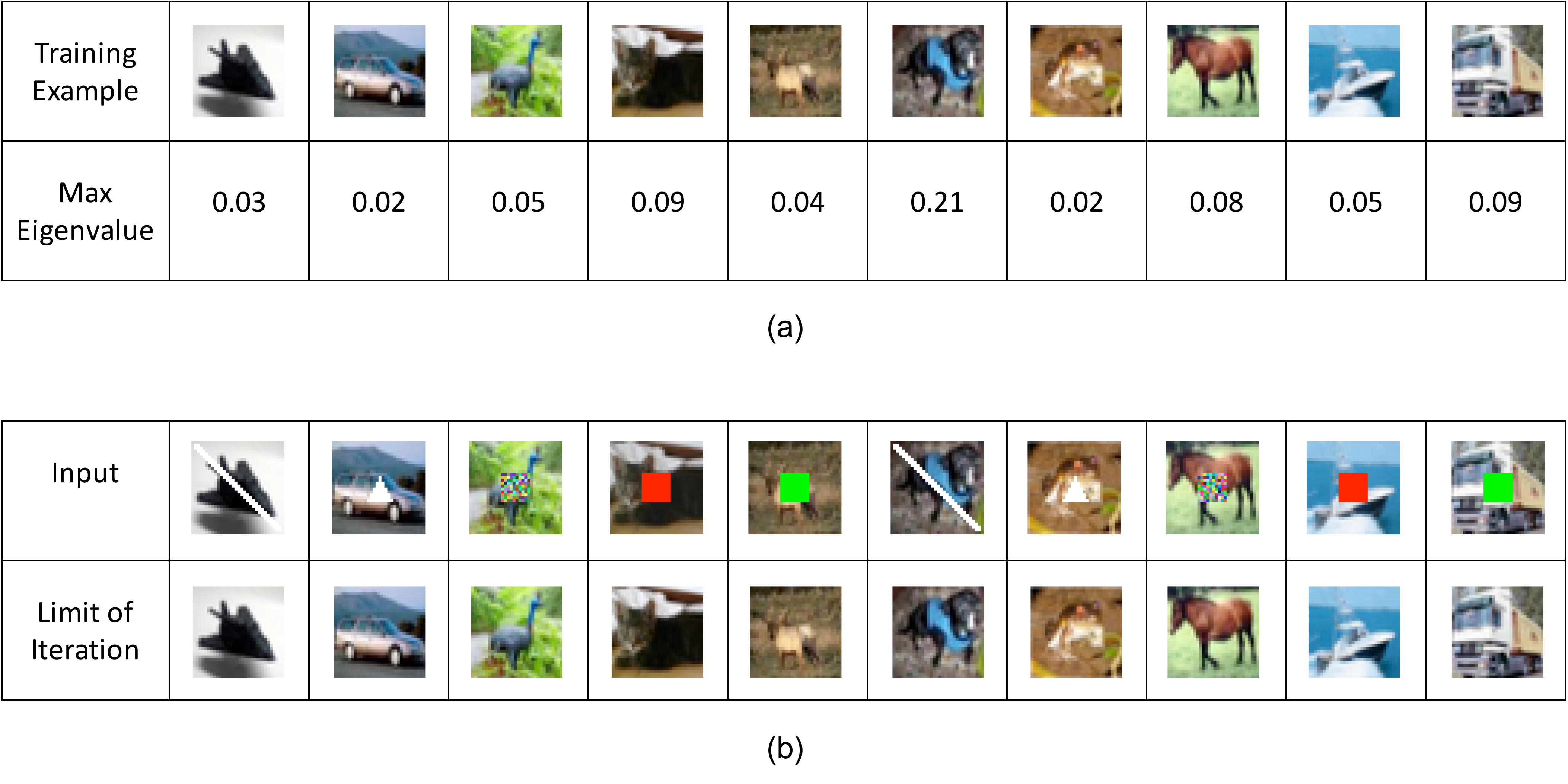}
\caption{A U-Net convolutional autoencoder \cite{UNet, DeepImagePrior} storing 10 CIFAR10 training examples as attractors. Training details are presented in Appendix Figure \ref{appendix: fig 6}b. (a) Maximum eigenvalue of the Jacobian of the network at the training example.  (b) Iteration from corrupted inputs converges to a training example.  In general, we observe that convolutional autoencoders store training examples as attractors when the receptive field size of the hidden units in the network covers the entire image.  This can be achieved by increasing the stride of the filters, as is done in the U-Net used here.}
\label{appendix: fig 10}
\end{figure}

\begin{figure}
\centering
\includegraphics[height=2.5in]{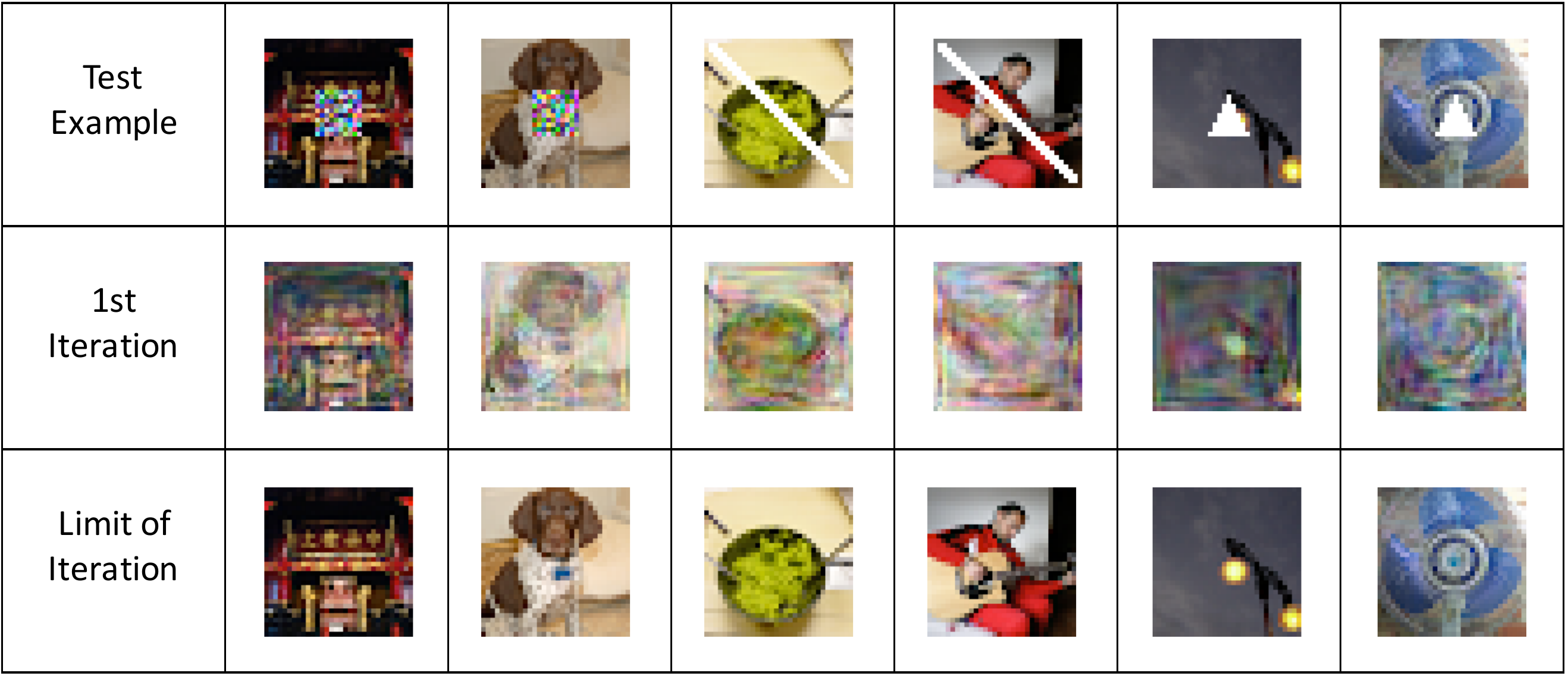}
\caption{A U-Net convolutional autoencoder \cite{UNet, DeepImagePrior} storing 100 ImageNet-32 training examples as attractors. Training details are presented in Appendix Figure \ref{appendix: fig 6}b. Iteration from corrupted inputs converges to a training example.}
\label{appendix: fig 11}
\end{figure}

\begin{figure}
\centering
\includegraphics[height=4.8in]{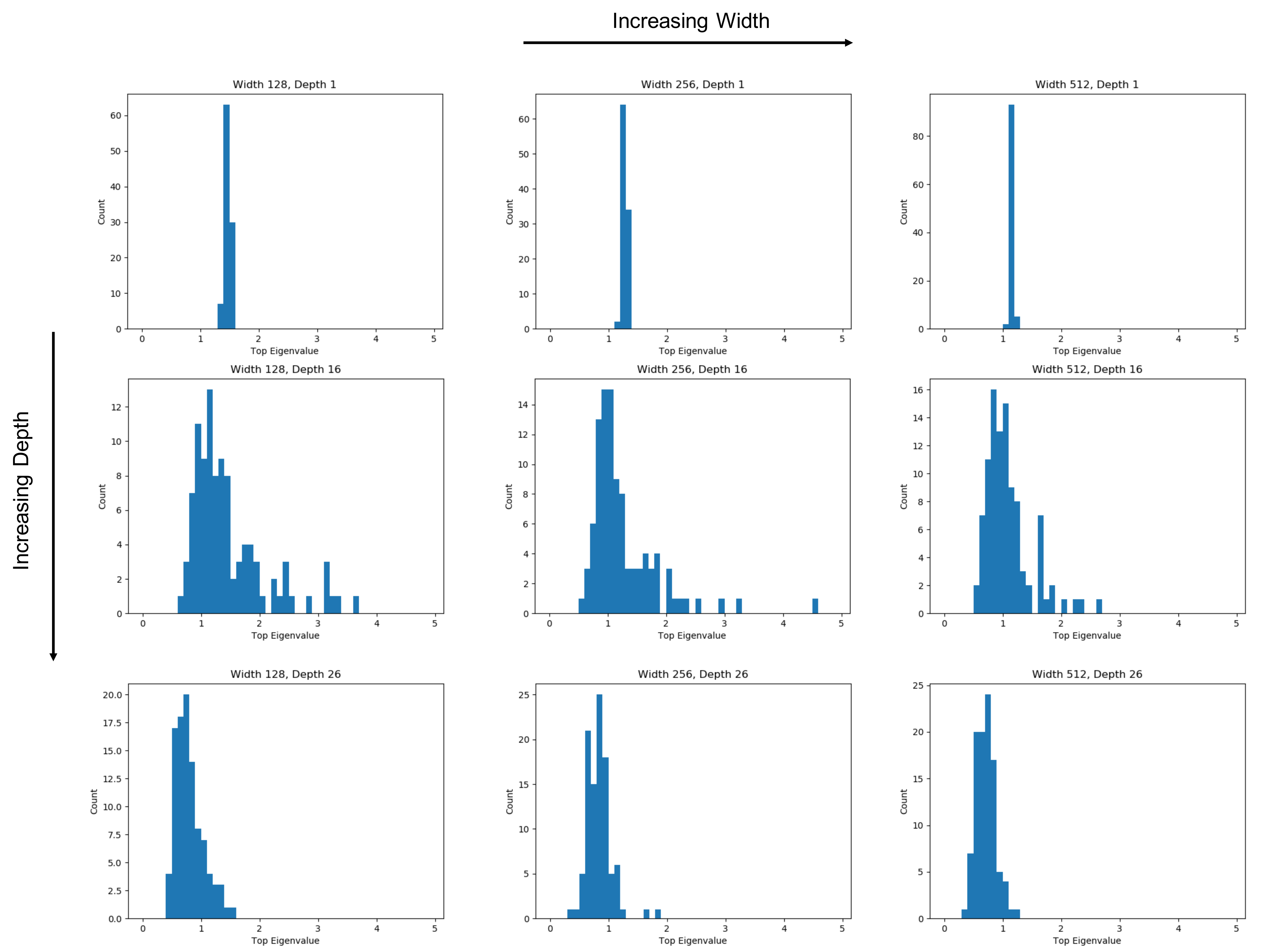}
\caption{Over-parameterized autoencoders become more contractive at the training examples as network depth and width are increased.  Networks are trained on 100 examples from MNIST and are a subset of architectures considered in Figure \ref{fig: Figure 5}a.  Histograms of the maximum eigenvalue of the Jacobian at each of the 100 training examples are presented for each of the nine settings.} 
\label{appendix: fig 12}
\end{figure}

\begin{figure}
\centering
\includegraphics[height=4.8in]{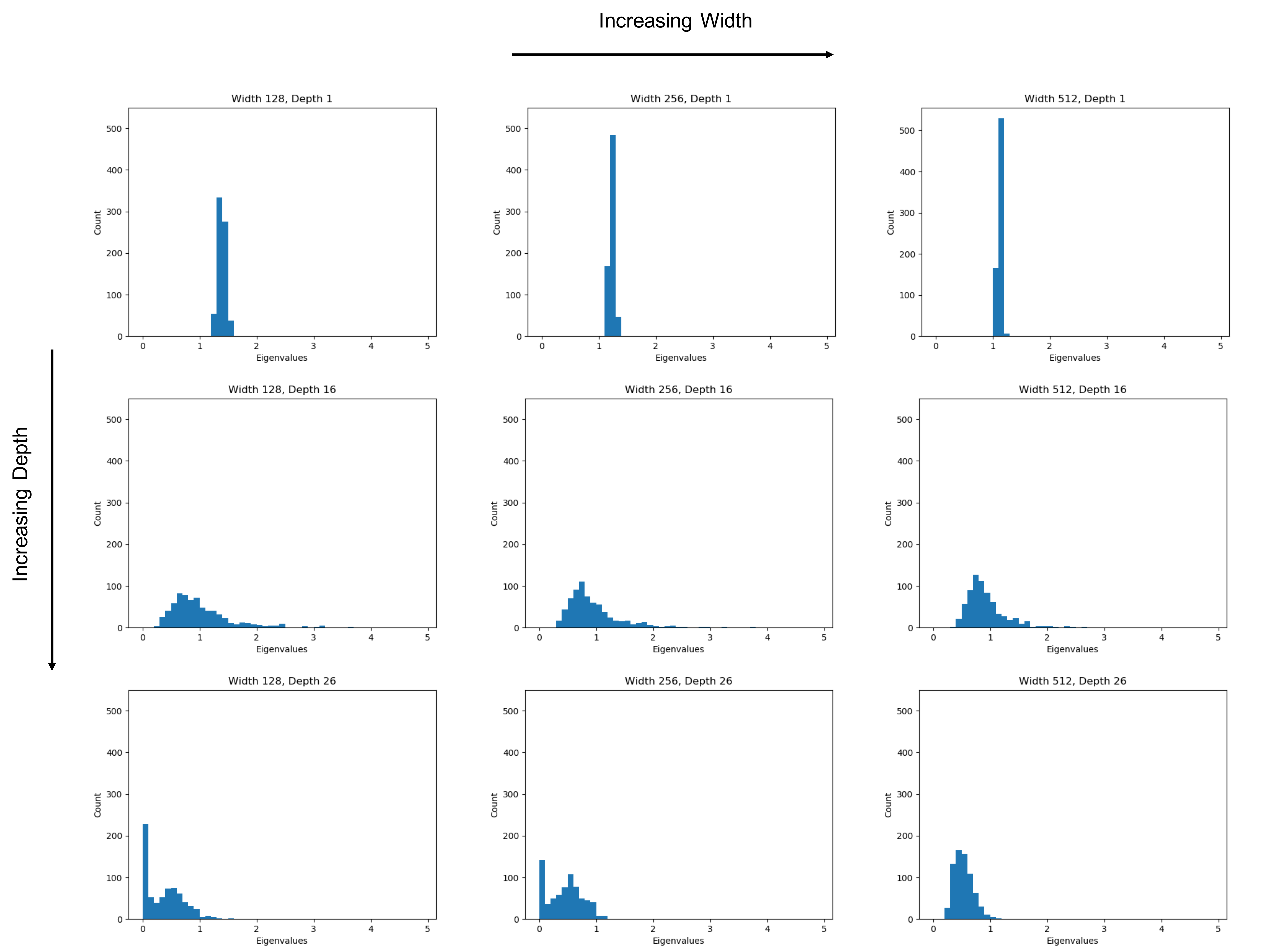}
\caption{Over-parameterized autoencoders become more contractive at the training examples as network depth and width are increased.  Networks are trained on 100 examples from MNIST and are a subset of architectures considered in Figure \ref{fig: Figure 5}a.  Histograms of top 1\% of $28^2$ eigenvalues of the Jacobian at each of the 100 training examples are presented for each of the nine settings. The variance of the distribution of eigenvalues decreases as width increases.} 
\label{appendix: fig 13}
\end{figure}

\begin{figure}
\centering
\includegraphics[height=2in]{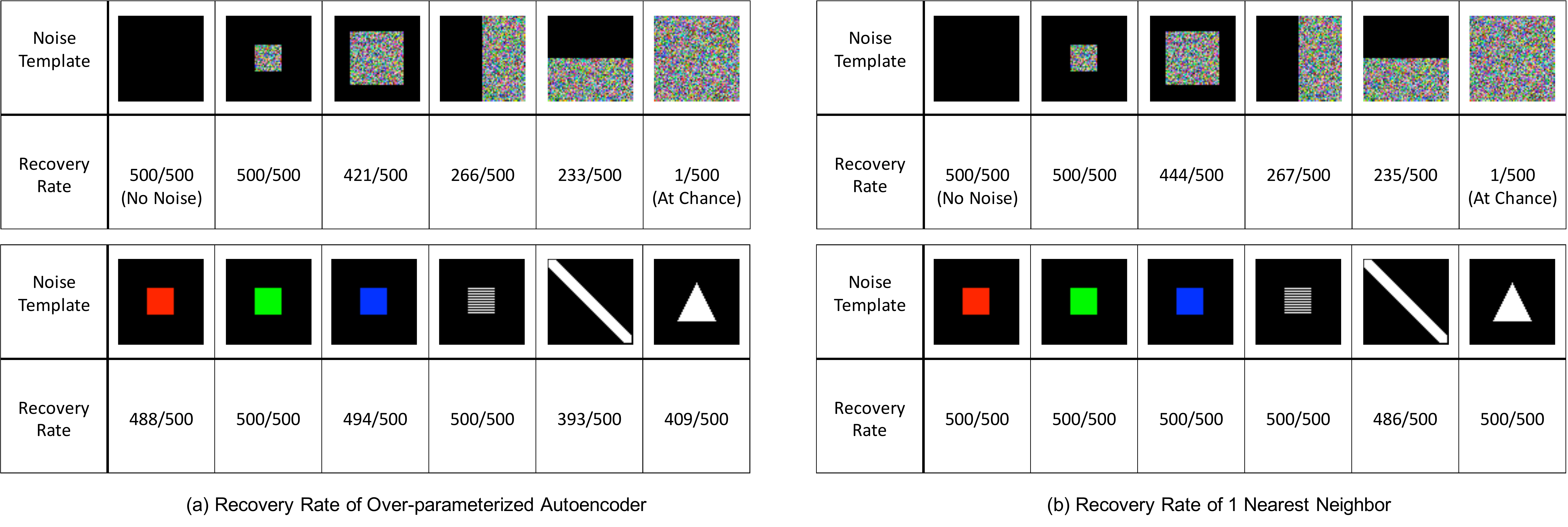}
\label{fig: Nearest Neighbor Comparison}
\caption{Comparison of recovery rate of over-parameterized autoencoder trained on 500 examples from ImageNet-64 (Figure 1 of main text) and 1 nearest neighbor (1-NN).  From Figure 2 of the main text, we know that over-parameterized autoencoders use metrics other than Euclidean distance to construct basins of attraction around training examples.  However, in this high dimensional setting, the metric used by the over-parameterized autoencoder is similar to that of 1-NN, as is demonstrated by the similar recovery rates from varying corruption patterns.} 
\label{appendix: fig 14}
\end{figure}

\newpage
\begin{figure}
    \scriptsize
    \centering
    \begin{tabular}{|>{\centering}m{1.8cm} V{3} c | c| c| c| c| c|}
    \hline
    \backslashbox[22mm]{Opt.}{Act.} & ReLU & Leaky ReLU & SELU & Swish & $\cos x - x$ & $ x + \frac{\sin 10x}{5}$\\
    \Xhline{3\arrayrulewidth}
    GD & 28/100 & 34/100 & 10/100 & $\text{NA}^{*}$ & 5/100 & 19/100 \\
    \hline
    GD + Momentum & 14/100 & 23/100 & 10/100 & $\text{NA}^{*}$ & 2/100 & 21/100 \\
    \hline
    GD + Momentum + Weight Decay & $\text{NA}^{*}$ & $\text{NA}^{*}$ & 18/100 & $\text{NA}^{*}$ & 22/100 & $\text{NA}^{*}$\\
    \hline
    RMSprop & 97/100 & 98/100 & 100/100 & 49/100 & 100/100 & 100/100\\
    \hline
    Adam & 38/100 & 53/100 & 30/100 & 14/100 & 100/100 & 100/100\\
    \hline
    \end{tabular}
    \vspace{3mm}
    \caption{Impact of optimizer and nonlinearity on number of training examples stored as attractors.  In all experiments, we used a fully connected network with 11 hidden layers, 256 hidden units per layer, and default PyTorch initialization. $(*)~ \text{NA}$ indicates that the training error did not even decrease below $10^{-5}$ in 1,000,000 epochs. Although more attractors arise with adaptive methods and trignometric nonlinearities, attractors arise in all settings considered.  Note that we used a loss threshold of $10^{-5}$ for this table since the non-adaptive methods could not converge to $10^{-8}$ in 1,000,000 epochs.}
    \label{appendix: fig 15}
\end{figure}

\begin{figure}
    \scriptsize
    \centering
    \begin{tabular}{|c V{3} c|c|c|c|c|c|}
    \hline
    \backslashbox[21mm]{Init.}{Act.} & ReLU & Leaky ReLU & SELU & Swish & $\cos x - x$  & $x + \frac{\sin 10x}{5}$\\
    \Xhline{3\arrayrulewidth}
    U$[-0.01, 0.01]$ & 62/100 & 78/100 & 78/100 &  16/100 & 26/100 & 93/100\\ 
    \hline 
    U$[-0.02, 0.02]$ & 43/100 & 65/100 & 71/100 & 20/100 & 31/100 & 70/100 \\ 
    \hline 
    U$[-0.05, 0.05]$ & 55/100 & 55/100 & 29/100 & 32/100 & 100/100 & 89/100\\
    \hline 
    U$[-0.1, 0.1]$ & 36/100 & 43/100 & 13/100 & 30/100 & 100/100 & $\text{NA}^*$ \\
    \hline 
    U$[-0.15, 0.15]$ & 34/100 & 38/100 & 13/100 & 6/100 & 100/100 & $\text{NA}^*$\\
    \hline
    \end{tabular}
    \vspace{3mm}
    \caption{Impact of initialization on number of training examples stored as attractors.  In all experiments, we used a fully connected network with 11 hidden layers and 256 hidden units per layer trained using the Adam optimizer (lr=$10^{-4}$).  $(*)~ \text{NA}$ indicates that the training error did not decrease below $10^{-8}$ in 1,000,000 epochs.  Attractors arise under all settings for which training converged.  Generally, more attractors arise under smaller initializations or with trignometric nonlinearities.}
    \label{appendix: fig 16}
\end{figure}










\end{document}